\documentclass{article} 
\usepackage{iclr2026_conference,times}

\usepackage{amsmath,amsfonts,bm}

\def\1{\bm{1}}

\DeclareMathAlphabet{\mathsfit}{\encodingdefault}{\sfdefault}{m}{sl}
\SetMathAlphabet{\mathsfit}{bold}{\encodingdefault}{\sfdefault}{bx}{n}

\usepackage[hidelinks]{hyperref}
\hypersetup{
  pdftitle={Asymptotic Preservation and Uniform Accuracy of Diffusion and Flow-Matching Samplers},
  pdfauthor={Shiheng Zhang},
  pdfsubject={Asymptotic preservation and uniform accuracy of generative samplers},
  pdfkeywords={diffusion models, flow matching, asymptotic-preserving methods, uniform accuracy, terminal noise}
}
\usepackage{url}
\usepackage{amssymb}
\usepackage{amsthm}
\usepackage{booktabs}
\usepackage{array}
\usepackage{graphicx}
\usepackage{enumitem}
\usepackage{placeins}

\theoremstyle{plain}
\newtheorem{theorem}{Theorem}
\newtheorem{proposition}{Proposition}

\newtheorem{corollary}{Corollary}
\theoremstyle{definition}
\newtheorem{definition}{Definition}

\theoremstyle{remark}
\newtheorem{remark}{Remark}

\newcommand{\smax}{\sigma_{\max}}
\newcommand{\Wtwo}{W_{2}}
\newcommand{\ExpE}{\mathbb{E}}
\newcommand{\dd}{\mathrm{d}}
\newcommand{\damb}{d_0}
\newcolumntype{L}[1]{>{\raggedright\arraybackslash}p{#1}}

\title{Asymptotic Preservation and \\
Uniform Accuracy of Diffusion and \\
Flow-Matching Samplers}

\author{Shiheng Zhang \\
University of Washington \\
\texttt{shzhang3@uw.edu}}

\iclrfinalcopy 
\begin{document}

\maketitle
\lhead{}\rhead{}\renewcommand{\headrulewidth}{0pt}

\begin{abstract}
Diffusion and Gaussian-interpolant flow-matching samplers approach data through
a terminal noise floor $\varepsilon$, a singular limit for manifold-supported
or rank-deficient data. We study two properties of a complete sampler
specification, comprising its update rule, time grid, and terminal rule.
Asymptotic preservation (AP) means a stable and consistent zero-noise
discretization with a step count bounded independently of $\varepsilon$.
Uniform accuracy (UA) of order $p$ means that, at numerical resolution $h$, the
endpoint $W_2$ error is $O(h^p)$ with a floor-independent constant. Bounded
log-noise stepping fails AP because
its step count diverges. Stopping a stable base solver at a positive switching
scale $a$ and appending one map fitted to the analytic normal mode restores AP.
On smooth compact boundaryless manifolds, the standard map has exact-input
error $O(a^2-\varepsilon^2)$ and sharp zero-floor error $\Theta(a^2)$. A base
solver with a floor-uniform order-$p$ estimate on the resolved interval retains
that order when $a=O(h^{p/2})$, provided the terminal transfer factor remains
bounded.
Along exact trajectories, the posterior-mean identity
$D(x(\sigma),\sigma)=x(\sigma)-\sigma x'(\sigma)$ cancels the linear terminal
defect and enables higher-order fitted maps. A three-evaluation Hermite
construction is uniformly third order for exact switching-scale input over
$0\le\varepsilon\le a$, and a
seven-evaluation construction is fourth order at zero. We classify
representative diffusion and flow-matching specifications by AP and UA. On EDM
and Rectified Flow checkpoints, a paired decomposition separates base-integration
from terminal-completion error and predicts held-out same-seed endpoint errors.
\end{abstract}

\section{Introduction}
\label{sec:intro}
Diffusion and flow-matching models sample by integrating a learned field from
noise toward data, usually stopping at a positive terminal noise level before
a final denoising step. Throughout, $\sigma$ denotes the running noise scale and
$\varepsilon$ its requested terminal value, so $\sigma\in[\varepsilon,\smax]$.
Sampler implementations often denote this terminal parameter by
$\sigma_{\min}$; we use $\varepsilon$ throughout. In
the manifold-supported or rank-deficient regime, the denoiser Jacobian has
normal eigenvalues that vanish with $\varepsilon$, whereas the score Jacobian
has normal eigenvalues of size $\varepsilon^{-2}$. The floor is therefore a
singular parameter. Existing
convergence analyses typically impose score regularity uniformly in time
\citep{chen2023sampling,lee2022convergence} or retain a positive early-stopping
scale \citep{benton2024linear,debortoli2022manifold}. We instead study the
zero-floor limit of the numerical method.

Asymptotic-preserving (AP) numerics asks whether a fixed discretization has a
stable and consistent limit when a small physical scale is removed
\citep{jin1999efficient,jin2022ap,roos2008robust}. Diffusion sampling adds a
feature absent from the usual fixed-domain setting: the log-noise horizon
$\log(\smax/\varepsilon)$ diverges. A method can therefore maintain a small local
step by taking increasingly many updates. In this expanding-horizon setting,
the classical non-resolution requirement becomes a floor-independent step
count. We reserve \emph{uniform
accuracy} (UA) for a quantitative error bound independent of $\varepsilon$.

We use this definition to ask which existing samplers are AP and with what UA
order. A method name alone does not determine the answer. AP belongs to the
complete specification
\[
\text{update rule}\; +\; \text{time parametrization and grid}\; +\;
\text{terminal rule}.
\]
Here $h>0$ denotes the base solver's resolution parameter, such as its maximum
native step, and $a:=a(h)>0$ is a terminal switching scale selected from that
resolution.
For example, a DPM-Solver update marched to an ever smaller floor with bounded
log-noise steps is not AP because its step count diverges. The same update used
only on a positive interval $[a,\smax]$, followed by one fitted terminal map,
can be AP. The specification-level view
distinguishes implementations that share a solver formula but use different
cutoffs or final denoising rules.

Two mechanisms determine our sampler classification. Bounded local log-noise
steps without a terminal map require
$N\gtrsim\log(\smax/\varepsilon)$ updates. By contrast, using the
posterior-mean denoiser $D$
and normalized residual $\eta(x,a)=(x-D(x,a))/a$, the fitted update
\[
x\longmapsto D(x,a)+\varepsilon\,\eta(x,a)
\]
extends a stable and consistent base solver from its interval
$[a,\smax]$ to every lower terminal noise level. The required transfer
condition is that the base-solver error at $a$ vanishes after multiplication
by the denoiser's Lipschitz factor.

The quantitative question is then how much order survives this completion and
how $a$ should scale with the resolution $h$. On smooth closed manifolds,
the exact-input fitted map has finite-floor error $O(a^2-\varepsilon^2)$ and
sharp zero-floor error $\Theta(a^2)$. Thus the usual one-evaluation terminal
rule imposes a genuine second-order bottleneck. We show how to remove it. The
identity $D(x(\sigma),\sigma)=x(\sigma)-\sigma x'(\sigma)$ cancels the linear
term of every smooth terminal trajectory. Scale extrapolation then recovers a
higher-order zero endpoint, and Hermite reconstruction supplies every
intermediate floor. For exact switching-scale input, our concrete
three-evaluation map is uniformly third order; numerical base output inherits
this order under a weighted Lipschitz transfer condition. A seven-evaluation
map is fourth order at zero under its trajectory and stage-accuracy hypotheses.
The endpoint error then separates into a transported base-integration term and
a terminal-completion term. This decomposition determines how the switching
scale should vary with the numerical resolution and can be measured directly
on learned checkpoints.

On a rank-deficient Gaussian model, the terminal map is exact in normal
directions and has tangential error $O(a^2)$, so an order-$p$ base method
retains its order when $a=O(h^{p/2})$. Complementary bounded-horizon
parametrizations show that AP specifications can still have different UA
orders. Together, these results establish the AP
status and UA orders of the analyzed DDIM, DPM-Solver-1/2, EDM,
Euler--Maruyama, and Rectified Flow configurations. They also give conditional
completion results for DPM-Solver-3, DEIS, PNDM, DPM-Solver++, and UniPC.

We work with the posterior-mean field of the true noised law in
variance-exploding coordinates. In these coordinates, deterministic and
stochastic diffusion samplers and Gaussian-interpolant flow-matching samplers
traverse the same marginal path under different time parametrizations and
stochasticity levels. This common representation separates the update rule,
time grid, and terminal rule within one AP and UA framework.

Finite-budget schedule optimization and adaptive solver allocation choose
nodes or solver orders for a fixed terminal protocol
\citep{sabour2024align,jo2026formalizing}. Complementary work studies terminal
flow-matching trajectories, denoising as approximate projection, and adaptation
to manifold support
\citep{DBLP:conf/icml/WanWM025,permenter2024interpreting,kumar2026flow}.
Small-noise analysis of Rao--Blackwellized tangent denoising targets also
exhibits intrinsic and extrinsic manifold corrections \citep{rawal2026rao}.
Concurrent local analysis derives tangent-cone boundary-layer expansions for
Gaussian-smoothed singular measures near boundaries and corners, including
their score and log-Hessian fields \citep{brosse2026boundary}. Here the
geometric object is the global $W_2$ error of posterior-mean terminal
completion. We follow the terminal floor to zero and classify the complete
specification of update, time grid, and terminal rule.

Extrapolation has also been used to refine consecutive clean-sample estimates
in LA-DPM \citep{zhang2023lookahead} and to combine multiple positive-time ODE
solutions in RX-DPM \citep{choi2025rxdpm}. Our terminal-order construction
instead uses nodes that shrink with the switching scale, transports that remain
exact on the normal mode, and moment conditions that determine how many
positive-scale predictions recover higher floor-uniform order.

\paragraph{Contributions.}
(i) We define AP for the vanishing-noise problem and prove the
expanding-horizon obstruction for bounded log-noise steps.
(ii) We quantify the UA of fitted completion: the standard map has sharp
$\Theta(a^2)$ zero-floor error on smooth closed manifolds. We then derive
minimal-node fitted extrapolation, a three-evaluation terminal map that is
uniformly third order for exact switching-scale input, and a seven-evaluation
fourth-order zero-floor map under the trajectory conditions. A weighted
Lipschitz condition transfers the FE3 order from numerical base output.
(iii) We classify representative deterministic, stochastic, and flow-matching
specifications by AP and Gaussian UA order. On two learned checkpoints, a
same-seed decomposition recovers the base-integration and terminal-completion
rates and predicts held-out endpoint errors.

\section{Preliminaries: one flow, many samplers}
\label{sec:prelim}

Diffusion and Gaussian-interpolant flow-matching samplers can be expressed on
the same family of noised data distributions. We define this family and write
its deterministic, stochastic, and flow-matching traversals in common
coordinates. A denoiser or velocity is \emph{exact} when it is determined
by the true law $p_\sigma$.

\paragraph{Noised marginals and residual velocity.}
Let $X_0\sim p_0$ be data in $\mathbb R^{\damb}$, where $\damb$ is the ambient
dimension, and set $X_\sigma=X_0+\sigma Z$ with
$Z\sim\mathcal N(0,I)$. The resulting variance-exploding (VE) marginals are
$p_\sigma:=p_0*\mathcal N(0,\sigma^2I)$
\citep{song2021sde,karras2022edm}. Sampling moves from
$p_{\smax}\approx\mathcal N(0,\smax^2I)$, with $\smax$ above the data scale,
toward $p_0$ by decreasing $\sigma$.
The theory initializes from the true noised law $p_{\smax}$, thereby separating
the terminal discretization from the usual high-noise Gaussian approximation.

The exact \emph{denoiser}
\[
D(x,\sigma):=\ExpE[X_0\mid X_\sigma=x]
=x+\sigma^2\nabla\!\log p_\sigma(x)
\]
is the minimum-mean-square-error (MMSE) estimate of the clean state from the noisy
observation $X_\sigma=x$; the second equality is Tweedie's formula
\citep{robbins1956empirical,efron2011tweedie}. The probability-flow ordinary
differential equation (PF-ODE) is therefore
\begin{equation}
\label{eq:pfode}
\frac{\dd X}{\dd\sigma}
=-\sigma\,\nabla\!\log p_\sigma(X)
=\frac{X-D(X,\sigma)}{\sigma}
=:\eta(X,\sigma),
\qquad \sigma:\smax\downarrow\varepsilon .
\end{equation}
The normalized residual $\eta=(x-D)/\sigma$ is both the PF-ODE velocity and
the posterior noise $\ExpE[Z\mid X_\sigma=x]$, the noise-prediction variable
used in diffusion models \citep{ho2020denoising,song2021ddim}. Equivalently,
\begin{equation}
\label{eq:micromacro}
x=D(x,\sigma)+\sigma\,\eta(x,\sigma).
\end{equation}
The decomposition expresses the state as its posterior mean plus a residual of
size $\sigma$. For
$X_\sigma\sim p_\sigma$, conditional Jensen gives
$\ExpE\|\eta(X_\sigma,\sigma)\|^2\le \damb$, so $\eta$ is uniformly $O(1)$
per coordinate in $L^2$; this bound is sharp in the normal directions.
Written as the bare score $-\sigma^{-1}\eta$, the same field has normal
amplitude $O(\sigma^{-1})$ and Jacobian $O(\sigma^{-2})$, so we formulate
the samplers in $D$ and $\eta$.

\paragraph{Deterministic and stochastic traversals.}
The PF-ODE exactly transports the family $\{p_\sigma\}$. A marginal-preserving
stochastic family is most simply written in log-noise time
$\lambda=\log(\smax/\sigma)$. We call a monotone choice of integration variable
a \emph{clock}. Set $\sigma_\lambda=\smax e^{-\lambda}$; $X_\lambda$ is the
state indexed by this clock, and $B_\lambda$ is standard Brownian motion in
$\lambda$-time. Then
\begin{equation}
\label{eq:sde}
\dd X_\lambda=-(1+\beta)\,\sigma_\lambda\,\eta(X_\lambda,\sigma_\lambda)\,\dd\lambda
+\sqrt{2\beta}\,\sigma_\lambda\,\dd B_\lambda,\qquad \beta\ge0.
\end{equation}
Here $\beta=0$ recovers the PF-ODE, while $\beta=1$ is the canonical reverse
stochastic differential equation (SDE)
\citep{anderson1982reverse,song2021sde}. The factor $1+\beta$ keeps the
marginals equal to $p_{\sigma_\lambda}$.

\paragraph{Flow matching in the same coordinates.}
Gaussian-path flow-matching and stochastic-interpolant models
\citep{lipman2023flowmatching,liu2023rectifiedflow,albergo2023interpolants,albergo2023stochastic}
use $X_t=\alpha_t X_0+\sigma_t Z$; rectified flow has
$\alpha_t=1-t$ and $\sigma_t=t$. Wherever $\alpha_t>0$, the scaling
$\tilde X=X_t/\alpha_t$ and $\sigma=\sigma_t/\alpha_t$ transforms this path
into the VE marginals above (cf.\ Proposition~2.2 of
\citealp{DBLP:conf/icml/WanWM025}). Its velocity is
\begin{equation}
v_t(x)=\dot\alpha_t\,D(\tilde x,\sigma)
      +\dot\sigma_t\,\eta(\tilde x,\sigma),
\qquad \tilde x=x/\alpha_t,\quad \sigma=\sigma_t/\alpha_t .
\label{eq:fm}
\end{equation}
Gaussian-interpolant diffusion and flow matching therefore differ only by a
scaling and a time parameterization; the same $D$ and $\eta$ determine both
exact velocities. Their sampler specifications differ through the
discretization and terminal rule.

\paragraph{Singular terminal regime and metric.}
As $\varepsilon\to0$, the split~\eqref{eq:micromacro} develops a singular terminal
regime. In
the exact linear-support model, write $x_\perp$ for the component normal to the
support. The posterior mean has zero normal component, so
$\dd x_\perp/\dd\sigma=x_\perp/\sigma$ and therefore
$x_\perp(\sigma)\propto\sigma$. Hence $x_\perp=\sigma\eta_\perp$ collapses
while $\eta_\perp$ remains order one. The same
normal scaling describes the leading local geometry near regular manifold
support \citep{kadkhodaie2024geometry}. Our analysis concerns this collapse of
the state with a nonvanishing normalized residual.

Because $p_\sigma$ may converge to a singular law, terminal convergence is
measured with the Wasserstein-2 metric
\begin{equation}
\Wtwo^2(\mu,\nu):=\inf_{\pi\in\Pi(\mu,\nu)}
\int_{\mathbb R^{\damb}\times\mathbb R^{\damb}}
\lVert x-y\rVert^2\,\dd\pi(x,y),
\label{eq:w2}
\end{equation}
for probability laws $\mu,\nu$ with finite second moments, where
$\Pi(\mu,\nu)$ is their set of couplings. The metric remains finite as
$p_\sigma$ collapses
onto lower-dimensional support, whereas total variation is maximal and
Kullback--Leibler divergence is infinite between a positive-$\sigma$
density and its singular limit
\citep{arjovsky2017towards,debortoli2022manifold}.

\providecommand{\norm}[1]{\left\lVert #1 \right\rVert}

\section{Asymptotic preservation and uniform accuracy}\label{sec:ap}

Fix a requested terminal floor $\varepsilon$ and a numerical resolution $h$,
defined as the maximum ordinary step in the chosen time variable.
Starting from $p_{\smax}$, the continuous dynamics in \eqref{eq:sde} have
endpoint law $p_\varepsilon$, and $p_\varepsilon\to p_0$ in $W_2$ as
$\varepsilon\to0$.

A numerical sampler specification consists of an update rule, a time
parametrization and grid rule, and a terminal rule. For a requested floor $\varepsilon$ and
resolution $h$, it starts from
$x^{(0)}\sim p_{\smax}$ and produces $x^{(0)},\ldots,x^{(N)}$. We denote the
distribution of its endpoint by $q_{h,\varepsilon}$, so that
$x^{(N)}\sim q_{h,\varepsilon}$; the run uses $N(h,\varepsilon)$ updates. Varying $h$ and
$\varepsilon$ while keeping the specification fixed gives the two-parameter
endpoint-law family $\{q_{h,\varepsilon}\}_{h,\varepsilon>0}$. Parenthesized superscripts
index numerical iterations; subscripts remain reserved for noise levels and
physical coordinates. At every fixed $\varepsilon>0$ we assume ordinary consistency,
$\Wtwo(q_{h,\varepsilon},p_\varepsilon)\to0$ as $h\to0$.

\begin{definition}[Asymptotic preservation]\label{def:apua}
A fixed sampler specification is AP with respect to $\varepsilon\to0$ when its
two-parameter family of runs satisfies the following four conditions. We write
$q_{h,\varepsilon}$ for the endpoint law of each run.
\begin{enumerate}[label=\textup{(\roman*)},leftmargin=2.1em,itemsep=2pt,topsep=3pt]
\item \emph{Uniform stability.} There exist $h_0,C>0$, independent of
$\varepsilon$, such that
\[
\sup_{0<\varepsilon\le\varepsilon_0}\sup_{0<h\le h_0}
\sup_{0\le n\le N(h,\varepsilon)}\ExpE\norm{x^{(n)}}^2\le C.
\]
The constant $C$ may depend on the model, $p_{\smax}$, and $\smax$ but is
uniform in $\varepsilon,n,N$; in particular, $h_0$ is floor-independent.
\item \emph{Floor-independent step count.} At every fixed $0<h\le h_0$,
$\sup_{0<\varepsilon\le\varepsilon_0}N(h,\varepsilon)<\infty$.
\item \emph{Discrete terminal limit.} At each fixed $h$, an endpoint law
$q_{h,0}$ exists and $\Wtwo(q_{h,\varepsilon},q_{h,0})\to0$ as $\varepsilon\to0$.
\item \emph{Consistency of the limit scheme.}
$\Wtwo(q_{h,0},p_0)\to0$ as $h\to0$.
\end{enumerate}
The last two conditions make the endpoint limits commute,
\begin{equation}
\lim_{h\to0}\,\lim_{\varepsilon\to0}\,q_{h,\varepsilon}
\;=\;p_0\;=\;
\lim_{\varepsilon\to0}\,\lim_{h\to0}\,q_{h,\varepsilon} .
\label{eq:commute}
\end{equation}
\end{definition}

For the fixed-stage specifications considered here, each update uses a
floor-independent number of denoiser or velocity evaluations. The step-count
condition therefore bounds the total number of neural-field evaluations up to
a method-dependent constant.

AP is a qualitative statement about the zero-floor limit at fixed numerical
resolution. Uniform accuracy additionally quantifies how the endpoint-law error
decays with $h$, uniformly over the requested floor.

\begin{definition}[Uniform accuracy]\label{def:acc}
For the same sampler specification at resolution $h$:
\begin{enumerate}[label=\textup{(\roman*)},leftmargin=2.1em,itemsep=2pt,topsep=3pt]
\item The \emph{limit scheme} has order $p$ if
$\Wtwo(q_{h,0},p_0)\le K_0h^p$.
\item The endpoint-law family is \emph{uniformly accurate} of order $p$ if
\[
\sup_{0<\varepsilon\le\varepsilon_0}
\Wtwo(q_{h,\varepsilon},p_\varepsilon)\le Kh^p,
\]
with $K$ independent of $\varepsilon$. The resolution $h$ controls the ordinary
steps on the base-solver interval; a terminal map is part of the specification.
\end{enumerate}
The constants are model-dependent but floor-independent. For an AP
specification, this bound passes to the limit scheme as $\varepsilon\to0$ and also
controls every intermediate floor.
\end{definition}

\section{Fitted terminal maps}\label{sec:fitted}

The floor-independent step-count condition forces some update to cross an
unresolved terminal interval. Suppose a base solver is stable and consistent
on $[a,\smax]$, but the requested floor satisfies
$\varepsilon<a$. We construct one additional update that reaches any lower
floor while transferring the stability and accuracy established at $a$.
The update is \emph{fitted}: its coefficients reproduce the analytic
asymptotic mode of the terminal equation \citep{roos2008robust}.

\paragraph{The normal-mode equation selects the map.}
On the normal-mode test equation, $D=0$ and $\eta=x/\sigma$ is conserved, so exact
transport from $a$ to $\varepsilon$ maps $x$ to $(\varepsilon/a)x$. For a general denoiser,
the decomposition \eqref{eq:micromacro} at scale $a$ is
$x=D(x,a)+a\eta(x,a)$. Preserving the denoised component while holding
$\eta(x,a)$ fixed as the scale changes from $a$ to $\varepsilon$ gives
\begin{equation}
T_{a\to\varepsilon}(x)
:=D(x,a)+\varepsilon\eta(x,a)
=\tfrac{\varepsilon}{a}x+\big(1-\tfrac{\varepsilon}{a}\big)D(x,a),
\qquad T_{a\to0}=D(\cdot,a).
\label{eq:fitstep}
\end{equation}
We call this map \emph{fitted to the normal-mode equation}: on that equation it
reproduces the analytic map for every $0\le\varepsilon\le a$ without resolving the
intervening scales.
Because $T_{a\to\varepsilon}(x)=x+(\varepsilon-a)\eta(x,a)$, this is one explicit
$\sigma$-clock Euler step. Its affine form is exactly the deterministic
denoising diffusion implicit model (DDIM) update in VE/EDM
coordinates \citep{song2021ddim}, and the constant-field exponential-integrator
step underlying \citet{lu2022dpmsolver,zhang2023deis}. The normal-mode equation motivates the
map for a general exact posterior-mean denoiser.

For a probability law $q$, write $T_\#q$ for the distribution of $T(X)$
when $X\sim q$.

\paragraph{Completing a base solver.}
Fix a resolution $h$, choose a floor-independent terminal switching scale
$a=a(h)$,
and denote the base solver's output law at scale $\sigma\in[a,\smax]$ by
$q_{h,\sigma}^{\mathrm{base}}$. The composite sampler is
\[
\smax\xrightarrow{\text{base solver}}a
\xrightarrow{\ T_{a\to\varepsilon}\ }\varepsilon,
\qquad
q_{h,\varepsilon}:=
\begin{cases}
q_{h,\varepsilon}^{\mathrm{base}}, & a\le\varepsilon\le\smax,\\
(T_{a\to\varepsilon})_\#q_{h,a}^{\mathrm{base}}, & 0\le\varepsilon<a.
\end{cases}
\]
For $\varepsilon\ge a$, only the base solver is used. For $\varepsilon<a$,
Proposition~\ref{prop:terminal-transfer} applies with
$q=q_{h,a}^{\mathrm{base}}$.

\begin{theorem}[Fitted terminal-map completion]
\label{thm:fittedap}
Assume the following.

\emph{(H1)} Initialized from $p_{\smax}$, the base solver is
uniformly stable on $[a,\smax]$, with a bound independent of $h$ and the
chosen switching scale $a$. At every fixed $\sigma>0$, it is consistent:
$\Wtwo(q_{h,\sigma}^{\mathrm{base}},p_\sigma)\to0$ as $h\to0$.

\emph{(H2)} As $h\to0$, $a\downarrow0$; the exact
denoiser $D(\cdot,a)$ is globally $L_a$-Lipschitz; and
$L_a\Wtwo(q_{h,a}^{\mathrm{base}},p_a)\to0$.

\emph{(H3)} For $\varepsilon\ge a$, stopping the base solver at $\varepsilon$
uses no more updates than stopping it at $a$.

Then the composite sampler specification is AP, with
\[
q_{h,\varepsilon}\xrightarrow[\varepsilon\to0]{W_2}
q_{h,0}:=\big(D(\cdot,a)\big)_\#q_{h,a}^{\mathrm{base}},
\qquad
N(h,\varepsilon)\le N_{\mathrm{base}}(h,a)+1.
\]
Here $N_{\mathrm{base}}(h,a)$ is the number of base-solver updates needed to
reach $a$.
It also satisfies
\begin{equation}
\sup_{0<\varepsilon\le\smax}\Wtwo\big(q_{h,\varepsilon},p_\varepsilon\big)
\le\max\left\{
\sup_{\sigma\in[a,\smax]}\Wtwo(q_{h,\sigma}^{\mathrm{base}},p_\sigma),
L_a\Wtwo(q_{h,a}^{\mathrm{base}},p_a)+a\sqrt{\damb}\right\}.
\label{eq:uatransfer}
\end{equation}
The composite specification is therefore UA of order $p$ whenever the
right-hand side is $O(h^p)$. Proofs are in Appendix~\ref{app:fitted-proofs}.
\end{theorem}

The theorem isolates the interaction between the two stages. The base solver
enters only through its error at the switching scale,
$\Wtwo(q_{h,a}^{\mathrm{base}},p_a)$, amplified by the Lipschitz factor $L_a$;
the fitted map adds the terminal term $a\sqrt{\damb}$. This general term is
sufficient for AP but is not sharp enough to determine UA orders.

On a smooth closed manifold, the universal $O(a)$ posterior-variance term
sharpens to a finite-floor completion error that vanishes when
$\varepsilon=a$.

\begin{theorem}[Finite-floor smooth-manifold completion]
\label{thm:manifold-finite-floor}
Let $\mathcal M\subset\mathbb R^{\damb}$ be a compact, boundaryless, embedded
$C^4$ manifold of positive reach and positive dimension, and let
$p_0=f\,\mathrm{vol}_{\mathcal M}$ for a strictly positive
$f\in C^2(\mathcal M)$. There are constants $a_0,C>0$ such that, for every
$0\le\varepsilon\le a\le a_0$,
\begin{equation}
\Wtwo\big((T_{a\to\varepsilon})_\#p_a,p_\varepsilon\big)
\le C(a^2-\varepsilon^2).
\label{eq:manifold-finite-floor}
\end{equation}
If $T_{a\to\varepsilon}$ is globally $K_{a,\varepsilon}$-Lipschitz, then every
input law $q$ with finite second moment also satisfies
\[
\Wtwo\big((T_{a\to\varepsilon})_\#q,p_\varepsilon\big)
\le K_{a,\varepsilon}\Wtwo(q,p_a)+C(a^2-\varepsilon^2).
\]
The proof is in Appendix~\ref{app:quadratic-closure}.
\end{theorem}

At zero floor the exponent is sharp:
$\Wtwo((D(\cdot,a))_\#p_a,p_0)=\Theta(a^2)$ under the same assumptions
(Proposition~\ref{prop:manifold-quadratic}). Thus quadratic order is the actual
smooth closed-manifold rate, not only a Gaussian upper bound.

\paragraph{Recovering higher uniform order.}

The quadratic barrier comes from an identity that also shows how to remove it.
For a fixed state at scale $a$, let $x(s)$ denote the exact probability-flow
trajectory on $0\le s\le a$, with $x(a)=x$. Along this trajectory,
\begin{equation}
D(x(s),s)=x(s)-s x'(s).
\label{eq:terminal-identity}
\end{equation}
Indeed, this is just the probability-flow equation rearranged. If
$x(\cdot)\in C^p([0,a])$, Taylor expansion of \eqref{eq:terminal-identity}
gives
\begin{equation}
D(x(s),s)
=x(0)+\sum_{j=2}^{p-1}\frac{1-j}{j!}x^{(j)}(0)s^j+O(s^p).
\label{eq:terminal-denoiser-expansion}
\end{equation}
The linear term vanishes identically. This is the structural fact behind the
usual quadratic terminal map and the higher-order construction below.

We call the concrete third- and fourth-order constructions FE3 and FE4. For
FE3, let $D_a=D(x,a)$ and
$\eta_a=(x-D_a)/a$. One explicit-midpoint transport to $a/2$ is
\[
\widehat x_{1/2}
=x-\frac a2\eta\!\left(x-\frac a4\eta_a,\frac{3a}{4}\right),
\qquad D_{1/2}:=D(\widehat x_{1/2},a/2).
\]
The three-evaluation zero-floor prediction and its finite-floor Hermite
extension are
\begin{align}
z_0^{[3]}&=-\frac13D_a+\frac43D_{1/2},
\label{eq:fe3-zero}\\
T_{a\to\varepsilon}^{[3]}(x)
&=z_0^{[3]}+\frac{\varepsilon}{a}\big(x+D_a-2z_0^{[3]}\big)
+\left(\frac{\varepsilon}{a}\right)^2\big(z_0^{[3]}-D_a\big).
\label{eq:fe3-hermite}
\end{align}

\begin{theorem}[Higher-order fitted terminal maps]
\label{thm:fitted-order-recovery}
\begin{enumerate}[label=\textup{(\roman*)},leftmargin=2.1em,itemsep=3pt,topsep=3pt]
\item Let $p\ge2$. For any $p-1$ distinct positive fractions of the switching
scale, there is a unique weighted sum of denoiser predictions that cancels the
terms of orders $2,\ldots,p-1$ in
\eqref{eq:terminal-denoiser-expansion}. If the terminal trajectory has a
uniformly bounded $p$th derivative, the transported stage states are
$O(a^p)$-accurate, and the denoiser is uniformly Lipschitz between exact and
numerical stage states, then this sum approximates $x(0)$ with error $O(a^p)$.
No universal linear construction using fewer than $p-1$ positive-scale
predictions can achieve this cancellation. Square-integrable versions of the
same bounds imply $W_2$ error $O(a^p)$ for exact input $X_a\sim p_a$.

\item Suppose the FE3 endpoint satisfies
$\|z_0^{[3]}-x(0)\|\le C_0a^3$ and
$\sup_{0\le s\le a}\|x^{(3)}(s)\|\le M_3$. Then
\begin{equation}
\sup_{0\le\varepsilon\le a}
\big\|T_{a\to\varepsilon}^{[3]}(x)-x(\varepsilon)\big\|
\le\left(C_0+\frac{2M_3}{81}\right)a^3.
\label{eq:fe3-uniform-bound}
\end{equation}
The map returns $x$ at $\varepsilon=a$ and is exact for every pure normal
trajectory $x(\sigma)=C\sigma$. If the pointwise hypotheses have
square-integrable envelopes and the FE3 map is Lipschitz, the same coupling
transfers this estimate to an arbitrary input law. In particular, $a=O(h)$ and
a Lipschitz-amplified base error of $O(h^3)$ give floor-uniform endpoint error
$O(h^3)$.
\end{enumerate}
The weights, explicit constants, law-level bound, and minimum-node proof are in
Appendix~\ref{app:high-order-terminal}.
\end{theorem}

The two manifold results above operate at different regularity levels.
Theorem~\ref{thm:manifold-finite-floor} is a law-level quadratic statement,
whereas Theorem~\ref{thm:fitted-order-recovery} is trajectory-level and requires
uniform terminal derivatives together with accurate transported stages. The
sphere calculation below verifies these stronger hypotheses on a curved
singular family.

The endpoint hypothesis follows from
Theorem~\ref{thm:fitted-order-recovery} with $p=3$ when the midpoint transport
has local error $O(a^3)$. The finite-floor result then uses the exact endpoint
derivative $a x'(a)=x-D_a$; no additional denoiser evaluation is needed.

An analogous seven-evaluation construction, FE4, is fourth order at zero when
its third-order Runge--Kutta (RK3) stage transports satisfy
Theorem~\ref{thm:fitted-order-recovery}. Both
maps preserve the exact normal mode; FE3 additionally provides the proved
higher-order map at every nonzero floor. Appendix~\ref{app:high-order-terminal}
gives the FE4 formula and verifies both orders in closed form on a sphere.

Figure~\ref{fig:high-order-terminal} verifies the nonsymmetric case and the
finite-floor order. On the nonuniform circle, ordinary Heun completion remains
second order at the same three- and seven-evaluation costs. The fitted maps
have log--log slopes $3.24$ and $4.57$ over the tested range, consistent with
the proved third- and fourth-order rates; the Hermite map remains at least third
order at every tested floor ratio. Proofs and exact-model details are in
Appendix~\ref{app:high-order-terminal}.

\begin{figure}[!t]
\centering
\includegraphics[width=0.94\linewidth]{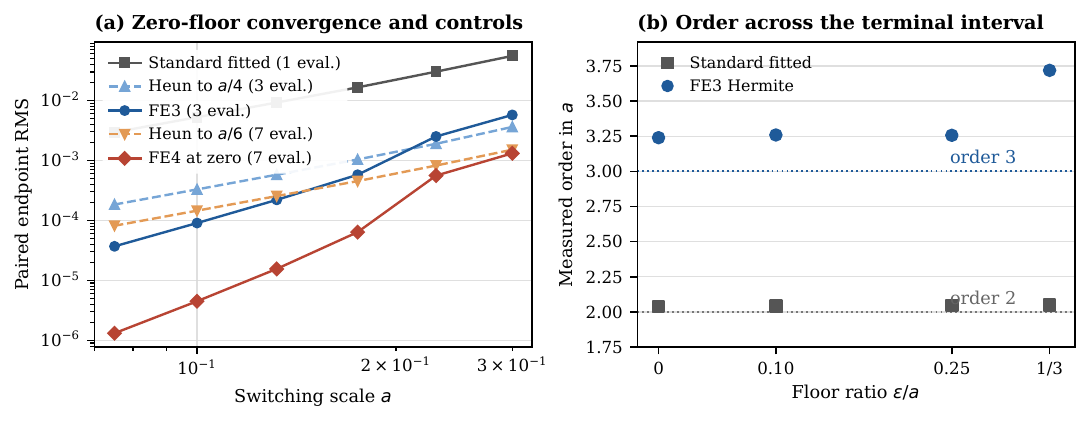}
\caption{Higher-order fitted terminal maps on the exact nonuniform-circle model.
(a) Paired zero-floor error against switching scale. FE3 and FE4 are compared
with ordinary positive-noise Heun completion at the same denoiser cost.
(b) Measured finite-floor order. The three-evaluation Hermite closure remains
third order across the unresolved terminal interval, while the standard fitted
map remains second order.}
\label{fig:high-order-terminal}
\end{figure}

For an order-$p$ base solver with a floor-uniform error estimate, balancing its
$O(h^p)$ error with the standard map's $O(a^2)$ error requires
$a=O(h^{p/2})$. The third-order FE3 map instead allows a third-order base
solver to use the natural scale $a=O(h)$.

Appendix~\ref{app:nfe-switching-proof} extends this balance to
switching-scale-dependent base estimates and fixed update budgets.

On a rank-deficient Gaussian target, the fitted map is exact in the normal
coordinates, nonexpansive in every coordinate, and has $O(a^2)$ true-input
error. The closed-form estimate is in Appendix~\ref{app:fitted-proofs} and is
used in the classification below.

\FloatBarrier

\section{Classification of existing samplers}
\label{sec:classification}

Table~\ref{tab:classification} classifies complete sampler specifications:
its columns mirror the three ingredients in Definition~\ref{def:apua}, followed
by the AP verdict and the Gaussian UA order. Changing the parametrization,
grid, or terminal rule can therefore change the verdict while leaving the
update rule unchanged. In particular, the first row applies to every update family
that is continued to $\varepsilon$ solely through bounded local log-noise
steps. A fitted terminal map is a sufficient completion mechanism, not a
necessary one: power-clock Euler and affine Rectified Flow reach the endpoint
on bounded native intervals without a separate fitted jump.

Because AP is problem-dependent, the table separates proved AP and UA results
on the rank-deficient Gaussian model of
Proposition~\ref{prop:gaussian-closure} from conditional results for the
remaining higher-order families. Theorem~\ref{thm:fittedap} supplies general
sufficient conditions for fitted specifications. The \emph{evolution
variable} is the independent variable in which the continuous sampler is
written: noise scale $\sigma$, log-noise time
$\lambda=\log(\smax/\sigma)$, power-noise time $\tau=\sigma^\gamma$, or
Rectified Flow's native time $t$. The next column specifies the grid and any
switching scale. For log-noise-controlled grids,
$\ell_n:=\log(\sigma_n/\sigma_{n+1})=\Delta\lambda_n$.

\begin{theorem}[Sampler AP and Gaussian UA classification]
\label{thm:sampler-classification}
On the rank-deficient Gaussian model, the specifications in
the first panel of Table~\ref{tab:classification} have the stated AP results
and UA orders. The second panel gives a conditional completion rule: if the
cited order-$p$ method has a floor-uniform error bound on $[a,\smax]$, then the
standard quadratic completion preserves that order under
$a=O(h^{p/2})$. FE3 instead preserves order three with the natural switching
scale $a=O(h)$ when the trajectory and stage-accuracy assumptions of
Theorem~\ref{thm:fitted-order-recovery} hold and the weighted input error
after Lipschitz amplification is $O(h^3)$. With the common switching scale
$a\asymp h$, the standard completion has guaranteed UA order $\min\{p,2\}$.
\end{theorem}

\FloatBarrier
\begin{table}[ht!]
\centering
\footnotesize
\setlength{\tabcolsep}{2.4pt}
\renewcommand{\arraystretch}{1.16}
\begin{tabular}{@{}L{0.18\linewidth}L{0.12\linewidth}L{0.22\linewidth}L{0.18\linewidth}L{0.06\linewidth}L{0.08\linewidth}@{}}
\toprule
Update rule or method & Evolution variable & Grid and switching scale & Terminal rule & AP & UA order \\
\midrule
\multicolumn{6}{@{}l}{\textbf{Proved on the rank-deficient Gaussian benchmark}}\\[2pt]
Any update rule
& Log-noise $\lambda$
& $\max_n\ell_n\le h$; continue to $\varepsilon$
& None; stop at $\varepsilon$
& no & -- \\
Explicit Euler
& Log-noise $\lambda$
& Fixed $N$; uniform grid to $\varepsilon$
& None; stop at $\varepsilon$
& no & -- \\
Explicit Euler
& Power noise $\tau=\sigma^\gamma$, $\gamma\ge1$
& Uniform grid to $\varepsilon$; $N=\lceil\smax^\gamma/h\rceil$
& Direct to $\varepsilon$
& yes & $1$ if $\gamma=1$; $1/\gamma$ otherwise \\
Frozen-$\eta$ Euler (DDIM, DPM-Solver-1, EDM Euler)
& Noise scale $\sigma$
& $\max_n\ell_n\le h$ on $[a,\smax]$; $a=O(h^{1/2})$
& Fitted $T_{a\to\varepsilon}$
& yes & $1$ \\
Midpoint exponential update (DPM-Solver-2)
& Log-noise $\lambda$
& $\max_n\ell_n\le h$ on $[a,\smax]$; $a=O(h)$
& Fitted $T_{a\to\varepsilon}$
& yes & $2$ \\
Heun predictor--corrector (EDM)
& Noise scale $\sigma$
& $\max_n\ell_n\le h$ on $[a,\smax]$; $a=O(h)$
& Fitted $\sigma$-Euler
& yes & $2$ \\
Euler--Maruyama, fixed $\beta>0$
& Log-noise $\lambda$
& $\max_n\ell_n\le h$ on $[a,\smax]$; $a=O(h^{1/2})$
& Fitted deterministic map
& yes & $1$ \\
Explicit Euler
& Rectified Flow time $t$
& $[0,t_{\max}]$; $\max_n|\Delta t_n|\le h$
& Direct to $t=0$
& yes & $1$ \\
Order-$p$ internally consistent Runge--Kutta
& Rectified Flow time $t$
& $[a_t,t_{\max}]$; switching scale $a_t=O(h^{p/2})$
& Fitted map to $t=0$
& yes & $p$ \\
\midrule
\multicolumn{6}{@{}l}{\textbf{Conditional fitted completion}}\\[2pt]
DPM-Solver-3, DEIS, PNDM, DPM-Solver++, or UniPC (order $p$)
& Method-specific
& Resolved interval corresponding to $[a,\smax]$; maximum native step $h$;
$a=O(h^{p/2})$
& Fitted $T_{a\to\varepsilon}$
& yes$^\ast$ & $p^\ast$ \\
Third-order base satisfying the FE3 transfer bound
& Method-specific
& Resolved interval corresponding to $[a,\smax]$; maximum native step $h$;
$a=O(h)$
& FE3 Hermite map \eqref{eq:fe3-hermite}
& yes$^\dagger$ & $3^\dagger$ \\
\bottomrule
\end{tabular}
\caption{Configuration-level exact-field classification. Each row fixes the
update rule, evolution variable, grid, and terminal rule before reporting
AP and its floor-uniform $W_2$ order on a fixed rank-deficient Gaussian
benchmark. The first row applies to every named update family under the stated
no-completion policy. Entries marked $^\ast$ are conditional on the
floor-uniform order-$p$ base estimate and the hypotheses of
Theorem~\ref{thm:fittedap}. The $^\dagger$ entry additionally assumes the
terminal trajectory, stage-accuracy, and weighted input-error conditions in
Theorem~\ref{thm:fitted-order-recovery}.}
\label{tab:classification}
\end{table}

\paragraph{Two ways AP fails.}
If every ordinary step has log-noise length
$\ell_n\le h$, telescoping the mesh gives
\[
\log(\smax/\varepsilon)=\sum_{n=0}^{N-1}\ell_n\le Nh.
\]
Without a terminal update that crosses the unresolved interval, the step-count
condition in Definition~\ref{def:apua} fails. This includes any
positive-cutoff protocol extended to smaller floors solely by adding bounded
log-noise steps. For plain $\lambda$-Euler, fixing $N$ produces
the other failure: the uniform step
$\ell=\log(\smax/\varepsilon)/N$ diverges, and its normal-mode multiplier
$(1-\ell)^N$ violates uniform stability. Bounded steps fail through a diverging
step count; stretched Euler steps fail through instability.

\begin{figure}[t]
\centering
\includegraphics[width=0.93\linewidth]{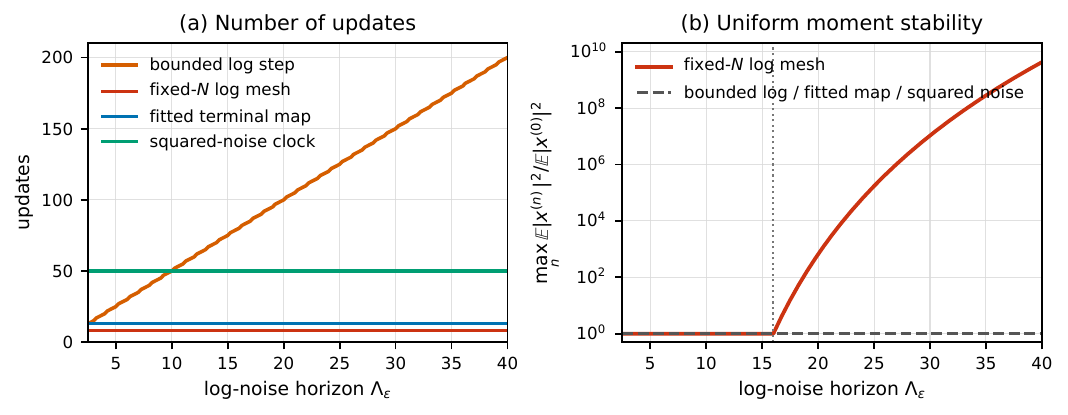}
\caption{The expanding-horizon obstruction on the normal-mode equation
$\dd x/\dd\lambda=-x$. (a) Keeping every log-noise step below $0.2$ is stable
but requires $N\asymp\Lambda_\varepsilon$ updates, where
$\Lambda_\varepsilon:=\log(\smax/\varepsilon)$. The other three displayed policies
have floor-independent counts. (b) Fixing $N=8$ instead stretches the Euler
step to $\Lambda_\varepsilon/N$; beyond the dotted threshold $\Lambda_\varepsilon/N=2$, its
maximum second moment diverges. Bounded-log stepping, the fitted terminal map,
and the squared-noise clock $\tau=\sigma^2$ remain moment-stable.
Their accuracy orders need not agree.}
\label{fig:ap-obstruction}
\end{figure}

\paragraph{Bounded-horizon clocks.}
On the Gaussian benchmark, the power clocks $\tau=\sigma^\gamma$,
$\gamma\ge1$, provide a family of AP
specifications that do not rely on a separate terminal map. Their native
horizon is bounded, so a uniform $\tau$-mesh has a floor-independent step
count. On the normal-mode test equation,
$\gamma=1$ is exact, whereas every $\gamma>1$ is stable but has sharp
floor-uniform order $1/\gamma$. The squared-noise clock is the case $\gamma=2$.
A bounded native horizon removes the step-count obstruction, but the clock
still determines how much terminal accuracy is retained.

\paragraph{First-order fitted methods.}
On the normal-mode test equation, $D=0$ and $\eta=x/\sigma$ is constant. A
$\sigma$-Euler step therefore sends $x(a)$ to $(\varepsilon/a)x(a)$ exactly. This is
the deterministic DDIM update and DPM-Solver-1
\citep{song2021ddim,lu2022dpmsolver}; it is also the final Euler step in the
EDM algorithms, where the Heun correction is explicitly skipped at zero
\citep{karras2022edm}. Among one-stage Euler clocks, affine functions of
$\sigma$ are the only clocks with this exact normal-mode property. The positive
Gaussian modes then give the first-order UA result.

\paragraph{Higher-order diffusion solvers.}
Midpoint DPM-Solver-2 has a proved Gaussian result. Its two stages are
exact on normal modes, while its tangential local defect is cubic in the
log step and concentrated where $\sigma^2$ crosses a data eigenvalue. This
defect sums to $O(h^2)$ independently of the floor. DPM-Solver-3, DEIS, PNDM,
DPM-Solver++, and UniPC improve the
base integration on $[a,\smax]$ using exponential, polynomial, multistep, or
predictor--corrector structure
\citep{lu2022dpmsolver,zhang2023deis,liu2022pndm,
lu2022dpmsolverplusplus,zhao2023unipc}. Their fixed-interval order describes
the base-solver interval. The conditional panel means: run the stated
order-$p$ method only to $a$, then append the terminal map \eqref{eq:fitstep}.
The fitted completion is then AP and preserves order $p$
when $a=O(h^{p/2})$. If $a\asymp h$, the $O(a^2)$ terminal-map error caps every
$p>2$ method at second order. EDM Heun is the $p=2$ case, for which
$a\asymp h$ satisfies this balance.

\paragraph{Stochastic and flow-matching samplers.}
The fitted theorem acts on endpoint laws and therefore also accepts a
stochastic base solver on $[a,\smax]$. On the Gaussian model, the
Euler--Maruyama covariance recurrence is uniformly first-order for each fixed
$\beta>0$; the deterministic fitted step then closes the remaining
floor. For affine flow matching,
$X_t=(1-t)X_0+tZ$, the finite high-noise level $\smax$ corresponds to
$t_{\max}=\smax/(1+\smax)<1$. The normal velocity is constant and this native
interval is bounded. Euler is first-order UA. Internally consistent Runge--Kutta methods retain
their classical order on the positive Gaussian modes, while the
fitted terminal map with an order-retaining switching scale handles the
singular normal endpoint.

\subsection{Learned-checkpoint uniform-order transfer}
\label{sec:learned-transfer}
The floor sweep reproduces the qualitative AP mechanisms in the classification.
Fitted or bounded-native-time policies showed floor-independent update counts,
bounded endpoint second moments, and convergent fixed-resolution terminal
outputs. Bounded log-noise stepping without a terminal map instead used
$O(\log(\smax/\varepsilon))$ updates.

A second experiment separates base integration from terminal-map error on an EDM
variance-preserving (VP) checkpoint and a Rectified Flow checkpoint. At switching
scale $a$, let $x_a^h$ and $x_a^\star$ denote the coarse and dense-reference
states, $x_0^\star$ the dense-reference endpoint, and $T_a$ the zero-floor
terminal map used by the checkpoint.
The same-seed endpoint error has the exact decomposition
\begin{equation}
T_a(x_a^h)-x_0^\star
=\big[T_a(x_a^h)-T_a(x_a^\star)\big]
+\big[T_a(x_a^\star)-x_0^\star\big].
\label{eq:checkpoint-decomposition}
\end{equation}
The first bracket is propagated base-integration error; the second is
terminal-map error on the dense trajectory. We fit their root-mean-square
magnitudes separately and retain their measured correlation when reconstructing
the total error. All 64 runs use matched initial noise and switching scale.

The estimated base-integration orders are $1.967$
for EDM and $1.926$ for Rectified Flow; the corresponding terminal-map
exponents are $0.757$ and $0.898$. Over the tested range, both learned fields
therefore retain approximately second-order Heun integration at fixed switching scale. Their
slower terminal-map convergence accounts for the lower rates observed when the
step size and switching scale shrink together. Component models estimated on
nine configurations predict seven configurations from the
excluded finest-step row and smallest-switching-scale column. Without
refitting, the median and maximum relative errors are $1.34\%$ and $7.35\%$
for EDM, and $8.21\%$ and $31.54\%$ for Rectified Flow.
Figure~\ref{fig:checkpoint-classification} shows the bootstrap intervals and
held-out predictions. Appendix~\ref{app:checkpoint-study} gives the protocol,
provenance, and all fitted exponents and correlations.

\begin{figure}[ht!]
\centering
\includegraphics[width=\linewidth]{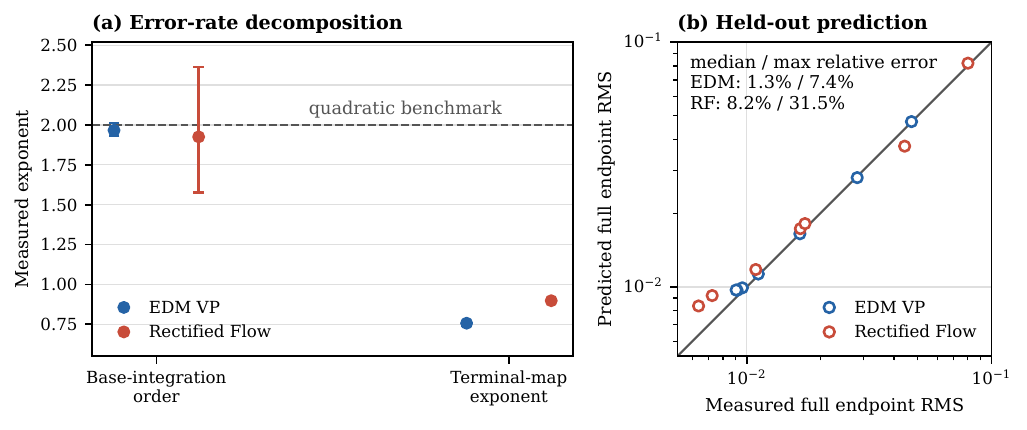}
\caption{Learned-checkpoint switching-scale decomposition. (a) Base-integration
order and terminal-map exponent; bars show 5--95\% paired-bootstrap intervals,
and the dashed line is the quadratic distributional $W_2$ benchmark for true
input. The checkpoint bars are exponents of same-seed endpoint root-mean-square
error. (b) Full endpoint
root-mean-square error predicted on configurations excluded from model fitting. Open circles show
held-out configurations and the diagonal is exact agreement.}
\label{fig:checkpoint-classification}
\end{figure}

\FloatBarrier
\section{Conclusion}
\label{sec:conclusion}
The vanishing terminal floor is a singular numerical limit governed by the
complete sampler specification, not by the update formula alone. Bounded local
log-noise stepping without a terminal map is not AP. A stable, consistent base
solver on $[a,\smax]$ satisfying the weighted transfer condition becomes AP
after one fitted terminal map.

The standard fitted construction also gives a quantitative order rule. On
smooth closed manifolds, its true-input error is
$O(a^2-\varepsilon^2)$ at every lower floor and sharply $\Theta(a^2)$ at
zero. The trajectory identity
$D(x(\sigma),\sigma)=x(\sigma)-\sigma x'(\sigma)$ exposes this
quadratic term and makes it cancellable. FE3 combines two positive-scale
denoiser predictions with endpoint Hermite reconstruction to obtain third
order uniformly over $0\le\varepsilon\le a$ for exact switching-scale input;
its weighted Lipschitz estimate transfers this order from numerical base
output. FE4 obtains fourth order at the zero floor under its stage-accuracy
hypothesis. The sphere calculation and nonuniform-manifold oracle verify these
orders against equal-evaluation Heun controls.

At the endpoint, the slower of the base-integration and terminal-completion
errors generally controls the observed order. On a rank-deficient Gaussian
benchmark, the standard map is exact in normal directions and contributes
$O(a^2)$ tangential error. An order-$p$ base method therefore retains order
$p$ when $a=O(h^{p/2})$, while FE3 lets a third-order method satisfying its
weighted transfer condition use $a=O(h)$. The resulting specification-level
AP and UA table covers
representative deterministic, stochastic, and flow-matching sampler families
(Table~\ref{tab:classification}). The classification also has a constructive
consequence: when the base solver satisfies the transfer conditions, a fitted map
repairs a non-AP terminal rule without modifying integration on $[a,\smax]$.
Appendix~\ref{app:nfe-switching-proof} gives more detailed switching-scale
bounds.

The checkpoint study separates exact-field order from its learned-field
transfer. At fixed switching scale, both checkpoints show approximately
second-order integration over the tested range, while their terminal-map errors
decay more slowly. The
held-out predictions show that these two terms account for the observed
same-seed endpoint errors across the tested switching scales and step sizes.
The coupled endpoint rate therefore depends on the full switching-scale rule.

\bibliography{iclr2026_conference}

@article{karras2022edm,
  title={Elucidating the design space of diffusion-based generative models},
  author={Karras, Tero and Aittala, Miika and Aila, Timo and Laine, Samuli},
  journal={Advances in Neural Information Processing Systems},
  volume={35},
  pages={26565--26577},
  year={2022}
}

@inproceedings{song2021ddim,
  title={Denoising diffusion implicit models},
  author={Song, Jiaming and Meng, Chenlin and Ermon, Stefano},
  booktitle={International Conference on Learning Representations},
  year={2021}
}

@inproceedings{song2021sde,
  title={Score-based generative modeling through stochastic differential equations},
  author={Song, Yang and Sohl-Dickstein, Jascha and Kingma, Diederik P and Kumar, Abhishek and Ermon, Stefano and Poole, Ben},
  booktitle={International Conference on Learning Representations},
  year={2021}
}

@article{ho2020denoising,
  title={Denoising diffusion probabilistic models},
  author={Ho, Jonathan and Jain, Ajay and Abbeel, Pieter},
  journal={Advances in Neural Information Processing Systems},
  volume={33},
  pages={6840--6851},
  year={2020}
}

@article{lu2022dpmsolver,
  title={{DPM-Solver}: A fast {ODE} solver for diffusion probabilistic model sampling in around 10 steps},
  author={Lu, Cheng and Zhou, Yuhao and Bao, Fan and Chen, Jianfei and Li, Chongxuan and Zhu, Jun},
  journal={Advances in Neural Information Processing Systems},
  volume={35},
  pages={5775--5787},
  year={2022}
}

@article{lu2022dpmsolverplusplus,
  title={{DPM-Solver++}: Fast solver for guided sampling of diffusion probabilistic models},
  author={Lu, Cheng and Zhou, Yuhao and Bao, Fan and Chen, Jianfei and Li, Chongxuan and Zhu, Jun},
  journal={Machine Intelligence Research},
  volume={22},
  number={4},
  pages={730--751},
  year={2025},
  publisher={Springer},
  doi={10.1007/s11633-025-1562-4}
}

@inproceedings{liu2022pndm,
  title={Pseudo numerical methods for diffusion models on manifolds},
  author={Liu, Luping and Ren, Yi and Lin, Zhijie and Zhao, Zhou},
  booktitle={International Conference on Learning Representations},
  year={2022}
}

@article{zhao2023unipc,
  title={{UniPC}: A unified predictor-corrector framework for fast sampling of diffusion models},
  author={Zhao, Wenliang and Bai, Lujia and Rao, Yongming and Zhou, Jie and Lu, Jiwen},
  journal={Advances in Neural Information Processing Systems},
  volume={36},
  year={2023}
}

@inproceedings{zhang2023deis,
  title={Fast sampling of diffusion models with exponential integrator},
  author={Zhang, Qinsheng and Chen, Yongxin},
  booktitle={International Conference on Learning Representations},
  year={2023}
}

@inproceedings{lipman2023flowmatching,
  title={Flow matching for generative modeling},
  author={Lipman, Yaron and Chen, Ricky TQ and Ben-Hamu, Heli and Nickel, Maximilian and Le, Matt},
  booktitle={International Conference on Learning Representations},
  year={2023}
}

@inproceedings{liu2023rectifiedflow,
  title={Flow straight and fast: Learning to generate and transfer data with rectified flow},
  author={Liu, Xingchao and Gong, Chengyue and Liu, Qiang},
  booktitle={International Conference on Learning Representations},
  year={2023}
}

@article{albergo2023stochastic,
  title={Stochastic interpolants: A unifying framework for flows and diffusions},
  author={Albergo, Michael S and Boffi, Nicholas M and Vanden-Eijnden, Eric},
  journal={Journal of Machine Learning Research},
  volume={26},
  number={209},
  pages={1--80},
  year={2025}
}

@inproceedings{albergo2023interpolants,
  title={Building normalizing flows with stochastic interpolants},
  author={Albergo, Michael S and Vanden-Eijnden, Eric},
  booktitle={International Conference on Learning Representations},
  year={2023}
}

@inproceedings{chen2023sampling,
  title={Sampling is as easy as learning the score: theory for diffusion models with minimal data assumptions},
  author={Chen, Sitan and Chewi, Sinho and Li, Jerry and Li, Yuanzhi and Salim, Adil and Zhang, Anru R},
  booktitle={International Conference on Learning Representations},
  year={2023}
}

@inproceedings{benton2024linear,
  title={Nearly $d$-linear convergence bounds for diffusion models via stochastic localization},
  author={Benton, Joe and De Bortoli, Valentin and Doucet, Arnaud and Deligiannidis, George},
  booktitle={International Conference on Learning Representations},
  year={2024}
}

@article{lee2022convergence,
  title={Convergence for score-based generative modeling with polynomial complexity},
  author={Lee, Holden and Lu, Jianfeng and Tan, Yixin},
  journal={Advances in Neural Information Processing Systems},
  volume={35},
  pages={22870--22882},
  year={2022}
}

@article{debortoli2022manifold,
  title={Convergence of denoising diffusion models under the manifold hypothesis},
  author={De Bortoli, Valentin},
  journal={Transactions on Machine Learning Research},
  year={2022}
}

@article{jin2022ap,
  title={Asymptotic-preserving schemes for multiscale physical problems},
  author={Jin, Shi},
  journal={Acta Numerica},
  volume={31},
  pages={415--489},
  year={2022},
  publisher={Cambridge University Press}
}

@article{jin1999efficient,
  title={Efficient asymptotic-preserving ({AP}) schemes for some multiscale kinetic equations},
  author={Jin, Shi},
  journal={SIAM Journal on Scientific Computing},
  volume={21},
  number={2},
  pages={441--454},
  year={1999},
  publisher={SIAM}
}

@book{roos2008robust,
  title={Robust numerical methods for singularly perturbed differential equations: convection-diffusion-reaction and flow problems},
  author={Roos, Hans-G{\"o}rg and Stynes, Martin and Tobiska, Lutz},
  year={2008},
  publisher={Springer}
}

@article{efron2011tweedie,
  title={Tweedie's formula and selection bias},
  author={Efron, Bradley},
  journal={Journal of the American Statistical Association},
  volume={106},
  number={496},
  pages={1602--1614},
  year={2011},
  publisher={Taylor \& Francis}
}

@inproceedings{kadkhodaie2024geometry,
  title={Generalization in diffusion models arises from geometry-adaptive harmonic representations},
  author={Kadkhodaie, Zahra and Guth, Florentin and Simoncelli, Eero P. and Mallat, St{\'e}phane},
  booktitle={International Conference on Learning Representations},
  year={2024}
}

@article{anderson1982reverse,
  title={Reverse-time diffusion equation models},
  author={Anderson, Brian DO},
  journal={Stochastic Processes and their Applications},
  volume={12},
  number={3},
  pages={313--326},
  year={1982},
  publisher={Elsevier}
}

@inproceedings{robbins1956empirical,
  title={An empirical {Bayes} approach to statistics},
  author={Robbins, Herbert E},
  booktitle={Proceedings of the Third Berkeley Symposium on Mathematical Statistics and Probability},
  volume={1},
  pages={157--163},
  year={1956},
  publisher={University of California Press}
}

@inproceedings{sabour2024align,
  title={Align Your Steps: Optimizing Sampling Schedules in Diffusion Models},
  author={Sabour, Amirmojtaba and Fidler, Sanja and Kreis, Karsten},
  booktitle={Proceedings of the 41st International Conference on Machine Learning},
  series={Proceedings of Machine Learning Research},
  volume={235},
  pages={42947--42975},
  year={2024}
}

@article{jo2026formalizing,
  title={Formalizing the Sampling Design Space of Diffusion-Based Generative Models via Adaptive Solvers and Wasserstein-Bounded Timesteps},
  author={Jo, Sangwoo and Choi, Sungjoon},
  journal={arXiv preprint arXiv:2602.12624},
  year={2026}
}

@article{arjovsky2017towards,
  title={Towards principled methods for training generative adversarial networks},
  author={Arjovsky, Martin and Bottou, L{\'e}on},
  journal={arXiv preprint arXiv:1701.04862},
  year={2017}
}

@inproceedings{DBLP:conf/icml/WanWM025,
  author={Zhengchao Wan and Qingsong Wang and Gal Mishne and Yusu Wang},
  title={Elucidating Flow Matching {ODE} Dynamics via Data Geometry and Denoisers},
  booktitle={Proceedings of the 42nd International Conference on Machine Learning},
  series={Proceedings of Machine Learning Research},
  volume={267},
  pages={62020--62083},
  year={2025}
}

@article{kumar2026flow,
  title={Flow matching is adaptive to manifold structures},
  author={Kumar, Shivam and Wang, Yixin and Lin, Lizhen},
  journal={arXiv preprint arXiv:2602.22486},
  year={2026}
}

@article{rawal2026rao,
  title={Rao-Blackwellized Score Matching on Manifolds},
  author={Rawal, Divit},
  journal={arXiv preprint arXiv:2605.25567},
  year={2026}
}

@article{brosse2026boundary,
  title={Boundary-Layer Asymptotics for Gaussian-Smoothed Singular Measures},
  author={Brosse, Nicolas and Dalalyan, Arnak S.},
  journal={arXiv preprint arXiv:2607.04514},
  year={2026}
}

@inproceedings{permenter2024interpreting,
  title={Interpreting and Improving Diffusion Models from an Optimization Perspective},
  author={Permenter, Frank and Yuan, Chenyang},
  booktitle={Proceedings of the 41st International Conference on Machine Learning},
  series={Proceedings of Machine Learning Research},
  volume={235},
  pages={40461--40483},
  year={2024},
  publisher={PMLR}
}

@article{leobacher2021projection,
  title={Existence, uniqueness and regularity of the projection onto differentiable manifolds},
  author={Leobacher, Gunther and Steinicke, Alexander},
  journal={Annals of Global Analysis and Geometry},
  volume={60},
  number={3},
  pages={559--587},
  year={2021},
  doi={10.1007/s10455-021-09788-z}
}

@book{hsu2002stochastic,
  title={Stochastic Analysis on Manifolds},
  author={Hsu, Elton P.},
  series={Graduate Studies in Mathematics},
  volume={38},
  publisher={American Mathematical Society},
  year={2002},
  doi={10.1090/gsm/038}
}

@article{erbar2010heat,
  title={The heat equation on manifolds as a gradient flow in the {Wasserstein} space},
  author={Erbar, Matthias},
  journal={Annales de l'Institut Henri Poincar{\'e}, Probabilit{\'e}s et Statistiques},
  volume={46},
  number={1},
  pages={1--23},
  year={2010},
  doi={10.1214/08-AIHP306}
}

@inproceedings{zhang2023lookahead,
  title={Lookahead Diffusion Probabilistic Models for Refining Mean Estimation},
  author={Zhang, Guoqiang and Niwa, Kenta and Kleijn, W. Bastiaan},
  booktitle={Proceedings of the IEEE/CVF Conference on Computer Vision and Pattern Recognition},
  pages={1421--1429},
  year={2023}
}

@inproceedings{choi2025rxdpm,
  title={Enhanced Diffusion Sampling via Extrapolation with Multiple {ODE} Solutions},
  author={Choi, Jinyoung and Kang, Junoh and Han, Bohyung},
  booktitle={International Conference on Learning Representations},
  year={2025}
}
\bibliographystyle{iclr2026_conference}

\appendix
\section{Proofs for the fitted completion}
\label{app:fitted-proofs}

\subsection{The general transfer estimate}

\begin{proposition}[One-step terminal transfer]\label{prop:terminal-transfer}
Let $D(x,\sigma)=\ExpE[X_0\mid X_\sigma=x]$ be the exact MMSE denoiser. Fix
$a>0$, let $q$ have finite second moment, and suppose $D(\cdot,a)$ is
globally $L_a$-Lipschitz. Then every $0\le\varepsilon\le a$ satisfies
\begin{equation}
\Wtwo\big((T_{a\to\varepsilon})_\#q,p_\varepsilon\big)
\le\left[\frac{\varepsilon}{a}+\left(1-\frac{\varepsilon}{a}\right)L_a\right]
\Wtwo(q,p_a)
+\left(1-\frac{\varepsilon}{a}\right)a\sqrt{\damb}.
\label{eq:fittedap}
\end{equation}
\end{proposition}

\begin{proof}[Proof of Proposition~\ref{prop:terminal-transfer}]
Let $Y=X_0+aZ\sim p_a$ and $X_\varepsilon=X_0+\varepsilon Z\sim p_\varepsilon$ share the
same pair $(X_0,Z)$. Since
$T_{a\to\varepsilon}(Y)=(1-\varepsilon/a)D(Y,a)+(\varepsilon/a)Y$,
\begin{equation}
T_{a\to\varepsilon}(Y)-X_\varepsilon
=-(1-\varepsilon/a)\bigl(X_0-D(Y,a)\bigr).
\label{eq:fitcoupling-new}
\end{equation}
The posterior mean is the $L^2$-optimal estimator of $X_0$ from $Y$;
comparison with the estimator $Y$ gives
\[
\ExpE\norm{X_0-D(Y,a)}^2
\le\ExpE\norm{X_0-Y}^2=\damb a^2.
\]
Thus
$\Wtwo((T_{a\to\varepsilon})_\#p_a,p_\varepsilon)
\le(1-\varepsilon/a)a\sqrt{\damb}$.
The map
$T_{a\to\varepsilon}=(1-\varepsilon/a)D(\cdot,a)+(\varepsilon/a)\,\mathrm{id}$ is
$[\varepsilon/a+(1-\varepsilon/a)L_a]$-Lipschitz. Push a near-optimal coupling of
$q$ and $p_a$ through this map and add the true-input term to obtain
\eqref{eq:fittedap}.
\end{proof}

\begin{proof}[Proof of Theorem~\ref{thm:fittedap}]
Set $a=a(h)$. Assumption \textup{(H3)} and the single fitted update give
\[
N(h,\varepsilon)\le N_{\mathrm{base}}(h,a)+1,
\]
uniformly in the requested floor. Assumption \textup{(H1)} bounds all base
iterates. To bound the fitted endpoint, couple
$X_q\sim q_{h,a}^{\mathrm{base}}$ and $X_p\sim p_a$. Posterior-mean
contraction and \textup{(H2)} give, for sufficiently small $h$,
\[
\ExpE\norm{D(X_q,a)}^2
\le2\ExpE\norm{X_0}^2
+2L_a^2\Wtwo^2(q_{h,a}^{\mathrm{base}},p_a)
\le2\ExpE\norm{X_0}^2+2.
\]
Because $T_{a\to\varepsilon}$ is a convex combination of $D(\cdot,a)$ and the
identity, the fitted state has a second-moment bound independent of $\varepsilon$.

At fixed $h$, the maps defining $q_{h,\varepsilon}$ and
$q_{h,0}=(D(\cdot,a))_\#q_{h,a}^{\mathrm{base}}$ differ by
$(\varepsilon/a)(x-D(x,a))$. Their identity coupling yields
\[
\Wtwo^2(q_{h,\varepsilon},q_{h,0})
\le(\varepsilon/a)^2
\ExpE_{q_{h,a}^{\mathrm{base}}}\norm{x-D(x,a)}^2\to0.
\]
At $\varepsilon=0$, Proposition~\ref{prop:terminal-transfer} gives
\[
\Wtwo(q_{h,0},p_0)
\le L_a\Wtwo(q_{h,a}^{\mathrm{base}},p_a)+a\sqrt{\damb}\to0.
\]
Together with fixed-floor consistency, these estimates prove all four AP
conditions. Finally, the right-hand side of \eqref{eq:fittedap} is affine in
$\varepsilon/a\in[0,1]$; taking its maximum at the two endpoints and combining it
with the base-solver interval supremum proves \eqref{eq:uatransfer}.
\end{proof}

\subsection{Sharp rate transfer and the fixed-update scale}
\label{app:nfe-switching-proof}

\begin{corollary}[General switching-scale rate bound]
\label{cor:general-rate}
Let
\[
m(a):=\Wtwo\big((D(\cdot,a))_\#p_a,p_0\big)
\]
be the zero-floor terminal-map error for true input law $p_a$.
Suppose, for some $p,\nu>0$ and $r\ge0$, that
$L_a\Wtwo(q_{h,a}^{\mathrm{base}},p_a)\lesssim h^p a^{-r}$ and
$m(a)\lesssim a^\nu$.
Then
\[
\Wtwo(q_{h,0},p_0)
\lesssim h^p a^{-r}+a^\nu.
\]
For $r>0$, balancing the two powers with
$a\asymp h^{p/(r+\nu)}$ gives
$\Wtwo(q_{h,0},p_0)=O(h^{p\nu/(r+\nu)})$.
For $r=0$, every $a=O(h^{p/\nu})$ retains order $p$.
For every exact posterior mean, $m(a)\le a\sqrt{\damb}$; on the Gaussian
benchmark, $m(a)=O(a^2)$.
\end{corollary}

\begin{proof}
Lipschitz pushforward and the triangle inequality give
$\Wtwo(q_{h,0},p_0)\le
L_a\Wtwo(q_{h,a}^{\mathrm{base}},p_a)+m(a)$. The universal estimate follows by
coupling $D(X_a,a)$ with $X_0$ and using posterior-mean optimality.
\end{proof}

Write $D_a:=D(\cdot,a)$ and define
\[
\bar\mu_a:=(D_a)_\#p_a,
\qquad
\mu_{h,a}:=(D_a)_\#q_{h,a}^{\mathrm{base}},
\]
\[
e_{\rm tr}:=\Wtwo(\mu_{h,a},\bar\mu_a),\qquad
e_{\rm cl}:=\Wtwo(\bar\mu_a,p_0),\qquad
e_{\rm end}:=\Wtwo(\mu_{h,a},p_0).
\]

\begin{theorem}[Sharp two-channel order transfer]
\label{thm:sharp-rate-transfer}
For every $h$ and $a$,
\begin{equation}
|e_{\rm tr}-e_{\rm cl}|
\le e_{\rm end}
\le e_{\rm tr}+e_{\rm cl}.
\label{eq:metric-sandwich}
\end{equation}
Let $a(h)=c h^\alpha$ and suppose, for $A,B,c>0$, $r\ge0$, and $\nu>0$, that
\[
e_{\rm tr}=Ac^{-r}h^{p-\alpha r}(1+o(1)),
\qquad
e_{\rm cl}=Bc^\nu h^{\alpha\nu}(1+o(1)).
\]
If the two exponents differ, the slower channel gives the full endpoint
asymptotic with the same leading constant. On the balance line
$p-\alpha r=\alpha\nu$,
\begin{equation}
|Ac^{-r}-Bc^\nu|
\le\liminf_{h\to0}\frac{e_{\rm end}}{h^{\alpha\nu}}
\le\limsup_{h\to0}\frac{e_{\rm end}}{h^{\alpha\nu}}
\le Ac^{-r}+Bc^\nu.
\label{eq:balanced-sandwich}
\end{equation}
Thus cancellation can alter the leading order only on the balance line.
\end{theorem}

\begin{theorem}[Bound-optimal fixed-update switching scale]
\label{thm:nfe-switching}
Let a fixed-stage base method use $M$ updates on a quasi-uniform log-noise mesh
from $\smax$ to $a$, and consider the zero-floor upper-bound model
\begin{equation}
\mathcal B_M(a)
=A M^{-p}\log(\smax/a)^p a^{-r}+Ba^\nu,
\qquad 0<a<\smax,
\label{eq:nfe-objective}
\end{equation}
where $p\ge1$, $\nu>0$, $r\ge0$, and $A,B>0$. For all sufficiently large
$M$, $\mathcal B_M$ has a unique interior minimizer $a_M$.
If $r>0$, then
\begin{equation}
a_M\sim
\left(\frac{Ar}{B\nu}\right)^{1/(r+\nu)}
\left(\frac{p}{r+\nu}\right)^{p/(r+\nu)}
M^{-p/(r+\nu)}(\log M)^{p/(r+\nu)},
\label{eq:nfe-scale-positive-r}
\end{equation}
and
\[
\min_a\mathcal B_M(a)
\asymp M^{-p\nu/(r+\nu)}(\log M)^{p\nu/(r+\nu)}.
\]
If $r=0$, then
\begin{equation}
a_M\sim
\left(\frac{Ap}{B\nu}\right)^{1/\nu}
\left(\frac p\nu\right)^{(p-1)/\nu}
M^{-p/\nu}(\log M)^{(p-1)/\nu},
\label{eq:nfe-scale-zero-r}
\end{equation}
and
\[
\min_a\mathcal B_M(a)
\sim A\left(\frac p\nu\right)^pM^{-p}(\log M)^p.
\]
\end{theorem}

For the smooth-manifold exponent $\nu=2$, a second-order base method with
$r=0$ has
\[
a_M\sim\sqrt{A/B}\,M^{-1}(\log M)^{1/2},
\qquad
\min_a\mathcal B_M(a)\sim A M^{-2}(\log M)^2.
\]

\begin{proof}[Proof of Theorem~\ref{thm:sharp-rate-transfer}]
The upper bound in \eqref{eq:metric-sandwich} is the triangle inequality, and
the lower bound is its reverse form. If
$p-\alpha r<\alpha\nu$, then $e_{\rm cl}/e_{\rm tr}\to0$; dividing the
sandwich by $e_{\rm tr}$ gives $e_{\rm end}/e_{\rm tr}\to1$. The case
$\alpha\nu<p-\alpha r$ is symmetric. On the balance line, divide the
sandwich by $h^{\alpha\nu}$ and pass to the lower and upper limits to obtain
\eqref{eq:balanced-sandwich}.
\end{proof}

\begin{proof}[Proof of Theorem~\ref{thm:nfe-switching}]
Set
\[
x=\log(\smax/a),\qquad a=\smax e^{-x}.
\]
Then minimizing \eqref{eq:nfe-objective} over $0<a<\smax$ is equivalent to
minimizing
\[
F_M(x)=A\smax^{-r}M^{-p}x^p e^{rx}
+B\smax^\nu e^{-\nu x},qquad x>0.
\]
For all sufficiently large $M$, the right derivative at zero is negative.
Moreover, $F_M(x)\to\infty$ as $x\to\infty$. Its critical points satisfy
\begin{equation}
e^{(r+\nu)x}x^{p-1}(p+rx)
=\frac{B\nu}{A}\smax^{r+\nu}M^p.
\label{eq:nfe-critical-point}
\end{equation}
The left-hand side is strictly increasing on $(0,\infty)$ for $p\ge1$.
Hence the critical point exists, is unique, and is the global minimizer.

Taking logarithms in \eqref{eq:nfe-critical-point} yields
\[
(r+\nu)x+(p-1)\log x+\log(p+rx)=p\log M+O(1).
\]
Consequently,
\[
x_M\sim\frac{p}{r+\nu}\log M\quad(r>0),
\qquad
x_M\sim\frac p\nu\log M\quad(r=0).
\]
Rearranging the critical-point equation gives
\[
B\nu a_M^{r+\nu}
=A M^{-p}x_M^{p-1}(p+rx_M).
\]
For $r>0$, $p+rx_M\sim rx_M$; substitution of the preceding asymptotic for
$x_M$ gives \eqref{eq:nfe-scale-positive-r}. The ratio of the base term in
$\mathcal B_M$ to the completion term tends to $\nu/r$, which gives the
stated order of the minimum. For $r=0$, the same identity gives
\eqref{eq:nfe-scale-zero-r}; the completion term is smaller than the base term
by the factor $p/(\nu x_M)$, and the displayed asymptotic for the minimum
follows.
\end{proof}

\subsection{The Gaussian terminal-map error is second order in the switching scale}

\begin{proposition}[Gaussian fitted-map error]
\label{prop:gaussian-closure}
Let $p_0=\mathcal N(0,C)$, with positive eigenvalues
$c_1,\ldots,c_d$ and any number of zero eigenvalues. For every law $q$ with
finite second moment and every $0\le\varepsilon\le a$,
\begin{equation}
\Wtwo\big((T_{a\to\varepsilon})_\#q,p_\varepsilon\big)
\le \Wtwo(q,p_a)
+\frac{a^2}{2}\left(\sum_{i=1}^{d}\frac1{c_i}\right)^{1/2}.
\label{eq:gaussian-closure}
\end{equation}
The fitted map is exact on the zero-eigenvalue normal coordinates and is
nonexpansive on every coordinate. For input law $p_a$ at zero floor,
\begin{equation}
m(a)^2=\sum_{i=1}^{d}
\left(\sqrt{c_i}-\frac{c_i}{\sqrt{c_i+a^2}}\right)^2,
\qquad
m(a)=\frac{a^2}{2}\left(\sum_{i=1}^{d}\frac1{c_i}\right)^{1/2}
+O(a^4).
\label{eq:gaussian-sharp-closure}
\end{equation}
\end{proposition}

\begin{proposition}[Sharp Gaussian order transfer under non-cancellation]
\label{prop:sharp-gaussian-transfer}
Suppose the base output at $a=a(h)$ is a centered Gaussian diagonal in the
eigenbasis of $C$, and let $s_{h,i}(a)$ be its standard deviation on a positive
mode $c_i>0$. Define
\[
b_{h,i}:=\frac{c_i}{c_i+a^2}
\left(s_{h,i}(a)-\sqrt{c_i+a^2}\right),
\qquad
g_{a,i}:=\frac{c_i}{\sqrt{c_i+a^2}}-\sqrt{c_i}.
\]
Then $\Wtwo(q_{h,0},p_0)=\norm{b_h+g_a}_2$. If
$b_h=h^\beta(b_0+o(1))$ and $g_{a(h)}=h^\gamma(g_0+o(1))$ with nonzero
leading vectors, the endpoint order is $\min\{\beta,\gamma\}$ when
$\beta\ne\gamma$. For $\beta=\gamma$, the same conclusion holds if
$b_0+g_0\ne0$. A sufficient finite-$h$ condition is
$\langle b_h,g_a\rangle\ge-\kappa\norm{b_h}\norm{g_a}$ for some $\kappa<1$.
\end{proposition}

\begin{proof}[Proof of Proposition~\ref{prop:gaussian-closure}]
Diagonalize $C$. On an eigenvector with eigenvalue $c\ge0$,
\[
D(x,a)=\frac{c}{c+a^2}x,
\qquad
T_{a\to\varepsilon}(x)=\frac{c+a\varepsilon}{c+a^2}x.
\]
For $c=0$, this map sends $\mathcal N(0,a^2)$ exactly to
$\mathcal N(0,\varepsilon^2)$. For $c>0$, its output from exact input $p_a$ has
standard deviation
\[
s_{\mathrm{fit}}=\frac{c+a\varepsilon}{\sqrt{c+a^2}},
\qquad
s_{\mathrm{exact}}=\sqrt{c+\varepsilon^2}.
\]
The difference-of-squares identity
\[
(c+a^2)(c+\varepsilon^2)-(c+a\varepsilon)^2=c(a-\varepsilon)^2
\]
implies
\[
0\le s_{\mathrm{exact}}-s_{\mathrm{fit}}
\le\frac{(a-\varepsilon)^2}{2\sqrt c}
\le\frac{a^2}{2\sqrt c}.
\]
Summing squared standard-deviation errors over the positive eigenspace gives
the second term in \eqref{eq:gaussian-closure}. The multiplier
$(c+a\varepsilon)/(c+a^2)$ lies in $[0,1]$ for $0\le\varepsilon\le a$; on a normal
coordinate it is $\varepsilon/a$. Hence $T_{a\to\varepsilon}$ is nonexpansive, and a
triangle inequality adds the input-law error $\Wtwo(q,p_a)$.
\end{proof}

\begin{proof}[Proof of the Gaussian completion statement in
Theorem~\ref{thm:sampler-classification}]
For $\varepsilon\ge a$, use the base-solver estimate. For $\varepsilon<a$, apply
\eqref{eq:gaussian-closure}. If $a\le C_a h^{p/2}$, both terms are
$O(h^p)$ uniformly in $\varepsilon$. If $a\asymp h$, the two powers are $h^p$ and
$h^2$, giving order $\min\{p,2\}$. AP follows from
Theorem~\ref{thm:fittedap}.
\end{proof}

\begin{proof}[Proof of Proposition~\ref{prop:sharp-gaussian-transfer}]
The zero-floor posterior-mean map multiplies positive mode $i$ by
$c_i/(c_i+a^2)$ and annihilates every zero mode. The completed standard
deviation error on a positive mode is therefore $b_{h,i}+g_{a,i}$. The closed
form of $W_2$ for commuting centered Gaussians gives the displayed identity.

If $\beta<\gamma$, then
$h^{-\beta}(b_h+g_a)\to b_0$; if $\gamma<\beta$, the symmetric argument gives
$g_0$. When the exponents agree,
$h^{-\beta}(b_h+g_a)\to b_0+g_0$. Continuity of the Euclidean norm proves the
three rate statements. Finally,
\[
\norm{b_h+g_a}^2
\ge(1-\kappa)\bigl(\norm{b_h}^2+\norm{g_a}^2\bigr)
\]
under the angle condition, by $2\norm{b_h}\norm{g_a}\le
\norm{b_h}^2+\norm{g_a}^2$.

The non-cancellation condition is necessary. In one positive mode, choose the
base standard deviation
$s_h(a)=\sqrt c\,(c+a^2)/c$. Its base and terminal defects are both
$\Theta(a^2)$, but posterior-mean completion returns standard deviation
exactly $\sqrt c$, so the endpoint error vanishes.
\end{proof}

\section{Proofs and exact-model checks for higher-order terminal maps}
\label{app:high-order-terminal}

\subsection{Endpoint expansion and minimum scale count}

Fix distinct ratios $\rho_i\in(0,1]$ and weights $w_i$ satisfying
\begin{equation}
\sum_iw_i=1,
\qquad
\sum_iw_i\rho_i^j=0\quad(2\le j\le p-1).
\label{eq:fitted-moment-conditions}
\end{equation}
Let $\widehat x_i$ approximate $x(\rho_i a)$ and define
\begin{equation}
\widehat T_a^{[p]}(x)
:=\sum_iw_iD(\widehat x_i,\rho_i a).
\label{eq:fitted-extrapolation}
\end{equation}
To make the constant in Theorem~\ref{thm:fitted-order-recovery} explicit,
assume
\[
\sup_{0\le s\le a}\|x^{(p)}(s)\|\le M_p,
\qquad
\|\widehat x_i-x(\rho_i a)\|\le C_Fa^p,
\]
and let $L_i$ be a Lipschitz constant for $D(\cdot,\rho_i a)$ between these
two stage states. Then the theorem gives
\begin{equation}
\big\|\widehat T_a^{[p]}(x)-x(0)\big\|
\le a^p\sum_i|w_i|
\left(L_iC_F+\frac{p+1}{p!}M_p\rho_i^p\right).
\label{eq:fitted-order-bound}
\end{equation}

\begin{proof}[Proof of Theorem~\ref{thm:fitted-order-recovery}]
Write $s_i=\rho_i a$ and $y_i=x(s_i)$. Taylor's theorem applied
separately to $x$ and $x'$ gives
\[
x(s)=\sum_{j=0}^{p-1}\frac{x^{(j)}(0)}{j!}s^j
+R_x(s),\qquad
x'(s)=\sum_{j=0}^{p-2}\frac{x^{(j+1)}(0)}{j!}s^j
+R_{x'}(s),
\]
where
$\norm{R_x(s)}\le M_ps^p/p!$ and
$\norm{R_{x'}(s)}\le M_ps^{p-1}/(p-1)!$. Substitution in
\eqref{eq:terminal-identity} yields
\[
D(x(s),s)
=x(0)+\sum_{j=2}^{p-1}\frac{1-j}{j!}
x^{(j)}(0)s^j+R_D(s),\qquad
\norm{R_D(s)}\le\frac{p+1}{p!}M_ps^p.
\]
The coefficient of $s$ is zero because the linear terms of $x$ and $s x'$ cancel.

Add and subtract $D(y_i,s_i)$ in \eqref{eq:fitted-extrapolation}. The
stage errors contribute at most
$a^p\sum_i|w_i|L_iC_F$. In the exact-stage terms, constant preservation
retains $x(0)$, the moment conditions cancel powers
$s^2,\ldots,s^{p-1}$, and the remainder contributes at most
\[
a^p\sum_i|w_i|\frac{p+1}{p!}M_p\rho_i^p.
\]
This proves \eqref{eq:fitted-order-bound}. If $X_a\sim p_a$, then the exact
endpoint $x(0)$ has law $p_0$; using these two random variables as a
coupling and averaging a square-integrable version of the bound proves the
$W_2$ statement.

It remains to prove the node claim. Suppose that only $k\le p-2$ distinct
positive ratios are used. Interpolate the values $1/\rho_i^2$ by a
polynomial $R$ of degree at most $k-1$, and put $Q(\rho)=\rho^2R(\rho)$.
Then $Q(\rho_i)=1$, while $Q$ belongs to the span of
$\rho^2,\ldots,\rho^{k+1}$. The cancellation conditions would give
$\sum_iw_iQ(\rho_i)=0$, whereas constant preservation gives the same sum
as $\sum_iw_i=1$, a contradiction. Hence $k\ge p-1$. For exactly
$p-1$ distinct positive ratios, the functions
$1,\rho^2,\ldots,\rho^{p-1}$ form a Chebyshev system on
$(0,\infty)$, so their generalized Vandermonde matrix is nonsingular and
the weights are unique.
\end{proof}

\subsection{Concrete FE3 and FE4 maps}

For fourth order at zero, use Kutta RK3 transports
$\widehat x_{1/2}^{\rm RK3}$ and $\widehat x_{1/3}^{\rm RK3}$ and set
\begin{equation}
z_0^{[4]}
=\frac1{12}D_a
-\frac43D(\widehat x_{1/2}^{\rm RK3},a/2)
+\frac94D(\widehat x_{1/3}^{\rm RK3},a/3).
\label{eq:fe4-zero}
\end{equation}
This uses seven denoiser evaluations and is fourth order whenever the RK3
stage transports meet the $O(a^4)$ hypothesis in
Theorem~\ref{thm:fitted-order-recovery}.

For the law-level part of Theorem~\ref{thm:fitted-order-recovery}, let
$\overline C_3$ be the $L^2$ constant obtained from
\eqref{eq:fe3-uniform-bound}. If $T_{a\to\varepsilon}^{[3]}$ is
$K_a^{[3]}$-Lipschitz uniformly over $0\le\varepsilon\le a$, then every input
law $q$ with finite second moment satisfies
\begin{equation}
\sup_{0\le\varepsilon\le a}
\Wtwo\big((T_{a\to\varepsilon}^{[3]})_\#q,p_\varepsilon\big)
\le K_a^{[3]}\Wtwo(q,p_a)+\overline C_3a^3.
\label{eq:fe3-law-transfer}
\end{equation}

\begin{proof}[Proof of the finite-floor FE3 statement in
Theorem~\ref{thm:fitted-order-recovery}]
Set $g(r)=x(ar)$ for $0\le r\le1$. The quadratic Hermite
interpolant determined by $g(0)$, $g(1)=x$, and
$g'(1)=a x'(a)=x-D_a$ is precisely \eqref{eq:fe3-hermite} with
$z_0^{[3]}$ replaced by $g(0)$. Replacing the endpoint by $z_0^{[3]}$
perturbs this polynomial by
$(1-r)^2(z_0^{[3]}-g(0))$. The Hermite remainder satisfies
\[
\norm{g(r)-H_2g(r)}
\le \frac{a^3M_3}{3!}\,r(1-r)^2
\le \frac{2M_3}{81}a^3,
\]
because $\max_{0\le r\le1}r(1-r)^2=4/27$. Since
$(1-r)^2\le1$, the triangle inequality proves
\eqref{eq:fe3-uniform-bound}. If $x(\sigma)=C\sigma$, then
$D=0$, $z_0^{[3]}=0$, and \eqref{eq:fe3-hermite} reduces to
$r x=C\varepsilon$. For the law-level statement, couple $q$ and $p_a$ nearly
optimally, push this coupling through $T_{a\to\varepsilon}^{[3]}$, and couple
its exact-input branch with the exact state $x(\varepsilon)\sim p_\varepsilon$.
The Lipschitz and $L^2$ bounds followed by the triangle inequality give
\eqref{eq:fe3-law-transfer}, uniformly in $\varepsilon$.
\end{proof}

For completeness, the Kutta transports in \eqref{eq:fe4-zero} can be
written explicitly. For $b\in\{1/2,1/3\}$, let
$\Delta_b=(b-1)a$ and
\[
k_1=\eta(x,a),\qquad
k_2=\eta\left(x+\frac{\Delta_b}{2}k_1,
a+\frac{\Delta_b}{2}\right),\qquad
k_3=\eta(x-\Delta_bk_1+2\Delta_bk_2,ba),
\]
\[
\widehat x_b^{\rm RK3}
=x+\frac{\Delta_b}{6}(k_1+4k_2+k_3).
\]
The explicit-midpoint FE3 map uses $D_a$, one midpoint evaluation, and one
evaluation at its transported output. The two FE4 branches share $D_a$; each
then uses two internal stage evaluations and one output evaluation. Hence FE3
and FE4 use three and seven denoiser evaluations, respectively. On the pure
normal equation $D=0$ and $x(\sigma)=C\sigma$, internal consistency makes
each Runge--Kutta stage exact, every denoiser prediction in the final
combination vanishes, and both zero-floor maps return the exact endpoint.

\subsection{Uniform-sphere calculation}

\begin{proposition}[Uniform-sphere terminal orders]
\label{prop:sphere-terminal-orders}
Let $d_0\ge2$, let $p_0$ be uniform on the radius-$R$ sphere in
$\mathbb R^{d_0}$, and let
$|x|=R+au$ with $u$ bounded. Then, uniformly on bounded $u$ sets,
\begin{align*}
z_0^{[3]}-R\frac{x}{|x|}
&=-\frac{(d_0-1)u}{12R^2}a^3\frac{x}{|x|}+O(a^4),\\
z_0^{[4]}-R\frac{x}{|x|}
&=\frac{(d_0-1)(d_0-3-4u^2)}{288R^3}a^4
\frac{x}{|x|}+O(a^5).
\end{align*}
The corresponding $L^2$ rates hold for $X_a=X_0+aZ$.
\end{proposition}

\begin{proof}[Proof of Proposition~\ref{prop:sphere-terminal-orders}]
Rotational symmetry makes the posterior on the sphere a von Mises--Fisher
law. If $\varrho=|x|$ and $\kappa=R\varrho/\sigma^2$, its posterior mean is
\[
D(x,\sigma)=R A_{d_0}(\kappa)\frac{x}{\varrho},\qquad
A_{d_0}(\kappa)=\frac{I_{d_0/2}(\kappa)}{I_{d_0/2-1}(\kappa)},
\]
where $I_\nu$ is the modified Bessel function. The probability-flow
trajectory is radial, so its zero-noise endpoint is $Rx/\varrho$. With
$m=d_0-1$, the standard large-$\kappa$ expansion is
\[
A_{d_0}(\kappa)
=1-\frac{m}{2\kappa}+\frac{m(m-2)}{8\kappa^2}
+O(\kappa^{-3}).
\]
Set $\varrho=R+au$ and substitute this expansion in the explicit-midpoint and
Kutta-stage formulas above. All stages remain radial. Collecting powers of
$a$ gives
\[
z_0^{[3]}-R\frac{x}{\varrho}
=-\frac{m u}{12R^2}a^3\frac{x}{\varrho}+O(a^4),
\]
\[
z_0^{[4]}-R\frac{x}{\varrho}
=\frac{m(m-2-4u^2)}{288R^3}a^4\frac{x}{\varrho}+O(a^5),
\]
uniformly for $u$ in a bounded set. These are the displayed formulas after
substituting $m=d_0-1$.

For $X_a=X_0+aZ$, the reverse triangle inequality gives
$\big||X_a|-R\big|/a\le\norm{Z}$, so the scaled radial displacement has moments of
every fixed order bounded uniformly in $a$. The Bessel remainder is uniform
on $|u|\le a^{-1/4}$ with polynomial dependence on $u$. On the complement,
the Gaussian tail is exponentially small, while the finite-stage maps grow at
most linearly in $|X_a|+R$. Truncation and dominated convergence therefore
upgrade the two pointwise expansions to the stated $L^2$ rates.
\end{proof}

\subsection{Exact-model protocol}

The nonuniform oracle in Figure~\ref{fig:high-order-terminal} uses the unit
circle density $(1+0.4\cos\theta)/(2\pi)$ and 8192 scrambled Sobol
samples for the zero-floor comparison. Six switching scales range from
$0.30$ to $0.075$. A dense RK4 probability-flow solve supplies the paired
reference and is checked by step doubling. FE3 is compared with one Heun step
to $a/4$ followed by denoising, so both use three denoiser evaluations; FE4 is
compared with three Heun steps to $a/6$, so both use seven. The finite-floor
sweep uses 4096 paired samples at
$\varepsilon/a\in\{0.1,0.25,1/3\}$. The measured FE3 orders are
$3.259$, $3.257$, and $3.719$, whereas the standard fitted map remains between
$2.041$ and $2.049$. The equal-evaluation Heun controls remain second order.

\section{Finite-floor and sharp smooth-manifold completion}
\label{app:quadratic-closure}

\begin{proposition}[Sharp smooth-manifold quadratic denoising error]
\label{prop:manifold-quadratic}
Let $\mathcal M\subset\mathbb R^{\damb}$ be a compact, boundaryless, embedded
$C^4$ manifold of positive reach and positive dimension, and let
$p_0=f\,\mathrm{vol}_{\mathcal M}$, where $f\in C^2(\mathcal M)$ is positive.
For the exact posterior-mean denoiser $D_a=D(\cdot,a)$, there are constants
$a_0,C,c>0$ such that
\[
m(a):=\Wtwo((D_a)_\#p_a,p_0),
\qquad
ca^2\le m(a)\le Ca^2,
\qquad 0<a\le a_0.
\]
Here $\Wtwo$ uses the ambient Euclidean cost.
\end{proposition}

\begin{proof}
Fix $0<r_{\mathrm{tube}}<\operatorname{reach}(\mathcal M)$. Let
$\operatorname{proj}_{\mathcal M}$ be the nearest-point projection on
$\{y:\operatorname{dist}(y,\mathcal M)<r_{\mathrm{tube}}\}$, extended outside
this tube by a bounded measurable $\mathcal M$-valued map, and set
$\mu_a=\operatorname{Law}(\operatorname{proj}_{\mathcal M}(X_0+aZ))$.  We compare the
denoised law to $p_0$ through $\mu_a$ and intrinsic heat flow $H_t$ generated
by $\tfrac12\Delta_{\mathcal M}$:
\begin{equation}
m(a)\le
\Wtwo((D_a)_\#p_a,\mu_a)
+\Wtwo(\mu_a,H_{a^2}p_0)
+\Wtwo(H_{a^2}p_0,p_0).
\label{eq:manifold-three-way}
\end{equation}

We first show that the first term is $O(a^2)$. For an observation $y$ in a
fixed smaller tube, write $z=\operatorname{proj}_{\mathcal M}(y)$ and $n=y-z$. Uniform
graph coordinates around $z$ have the form
$F_z(u)=z+u+g_z(u)$, where
$g_z(0)=Dg_z(0)=0$, $\|g_z(u)\|\le C\|u\|^2$, and
$\|Dg_z(u)\|\le C\|u\|$.  After removing the constant normal likelihood, the
posterior exponent is
\[
\phi_{z,n}(u)
=\frac{\|y-F_z(u)\|^2-\|n\|^2}{2}
=\frac{\|u\|^2}{2}-n\cdot g_z(u)+\frac{\|g_z(u)\|^2}{2}.
\]
Positive reach and compactness give uniform local coercivity
$\phi_{z,n}(u)\ge c\|u\|^2$ and a positive exponent gap off the chart. Hence the
full posterior conditioned on $X_a=y$ satisfies
$\mathbb E_y\|u\|^j\le C_j a^j$.  If
$A_z(u)=f(F_z(u))J_z(u)$ includes the volume Jacobian, then $A_z$ is uniformly
positive with bounded derivative.  Since
\[
\nabla_u\phi_{z,n}(u)
=u-Dg_z(u)^\top(n-g_z(u)),
\]
cutoff integration by parts gives
\[
\|\mathbb E_y u\|\le C(a^2+a\|n\|),
\qquad
\|\mathbb E_y g_z(u)\|\le Ca^2.
\]
The cutoff and off-chart terms are exponentially small relative to the local
normalizer.  Therefore
\begin{equation}
\|D_a(y)-\operatorname{proj}_{\mathcal M}(y)\|
\le C\bigl(a^2+a\,\operatorname{dist}(y,\mathcal M)\bigr).
\label{eq:posterior-projection}
\end{equation}
On $\{a\|Z\|<r_{\mathrm{tube}}/2\}$, the observation $y=X_0+aZ$ lies in the
tube and its distance to $\mathcal M$ is at most $a\|Z\|$. On the complement,
the probability is $O(e^{-c/a^2})$; both $D_a(y)$ and the extended projection
remain in a fixed bounded set because $\mathcal M$ is compact. Squaring
\eqref{eq:posterior-projection} and averaging therefore yields
\[
\mathbb E\|D_a(X_a)-\operatorname{proj}_{\mathcal M}(X_a)\|^2\le Ca^4.
\]
This same-observation coupling bounds the first term of
\eqref{eq:manifold-three-way} by $Ca^2$.

For the second term, put $t=a^2$ and let $B_t=aZ$.  Tubular projection
regularity \citep{leobacher2021projection} gives, uniformly for
$x\in\mathcal M$ on the event $\|B_t\|<r_{\mathrm{tube}}/2$,
\[
\operatorname{proj}_{\mathcal M}(x+B_t)=x+P_x B_t+R_t,
\qquad \mathbb E\|R_t\|^2\le Ct^2,
\]
The complementary Gaussian event has probability $O(e^{-c/t})$ and contributes
the same order because the projection extension and $\mathcal M$ are bounded.
Here $P_x$ is orthogonal projection onto $T_x\mathcal M$. Couple this step
to Brownian motion on $\mathcal M$ using the same ambient Brownian motion.  Its
standard ambient It\^o realization is
$\dd Y_s=P_{Y_s}\dd B_s+\tfrac12H(Y_s)\dd s$, where
$H=\Delta_{\mathcal M}\operatorname{id}_{\mathcal M}$ fixes the mean-curvature
sign convention and the generator is $\tfrac12\Delta_{\mathcal M}$
\citep[Chapters~2--3]{hsu2002stochastic}. It\^o isometry,
bounded geometry, and $\mathbb E\|Y_s-x\|^2\le Cs$ give
\[
\mathbb E\|Y_t-x-P_x B_t\|^2\le Ct^2.
\]
Thus $\Wtwo(\mu_a,H_{a^2}p_0)\le Ct=Ca^2$.

Finally write $H_s p_0=f_s\,\mathrm{vol}_{\mathcal M}$.  Positivity and $C^2$
regularity on the compact manifold bound the intrinsic Fisher information of
$f_s$ for small $s$.  The heat equation is a continuity equation with velocity
$-\tfrac12\nabla\log f_s$, so the Riemannian dynamic transport bound
\citep[Proposition~2.5]{erbar2010heat} gives
\[
\Wtwo(H_t p_0,p_0)
\le\frac12\int_0^t
\left(\int_{\mathcal M}\|\nabla\log f_s\|^2f_s\,\dd\mathrm{vol}_{\mathcal M}\right)^{1/2}
\dd s
\le Ct.
\]
The ambient chordal cost is no larger than the intrinsic geodesic cost.  If
$\mathcal M$ is disconnected, apply this argument on its finitely many
components and combine the componentwise couplings.  Taking $t=a^2$ closes
all three terms in \eqref{eq:manifold-three-way}.

It remains to prove the matching lower bound. Let
$\Sigma_a(y)=\operatorname{Cov}(X_0\mid X_0+aZ=y)$. Fix $R_0>0$. Uniformly
for observations with $\operatorname{dist}(y,\mathcal M)\le R_0a$, the normal
graph above has $n=aw$ with $\|w\|\le R_0$. After the rescaling $u=av$, its
posterior density is proportional to
\[
\exp\!\left[-\frac12\|v\|^2
+\frac{n\cdot g_z(av)}{a^2}
-\frac{\|g_z(av)\|^2}{2a^2}\right]A_z(av).
\]
The uniform graph bounds, coercivity, and off-chart phase gap give an
integrable Gaussian envelope. Since $g_z(av)=O(a^2\|v\|^2)$ and $f$ is
uniformly positive, the local normalizers and their first two moments converge
uniformly in $z$ and $\|w\|\le R_0$. The off-chart contribution is uniformly
exponentially small. Returning through
$F_z(u)=z+u+O(\|u\|^2)$ therefore gives
\[
\sup_{\substack{z\in\mathcal M,\ w\in N_z\mathcal M\\ \|w\|\le R_0}}
\left\|a^{-2}P_z\Sigma_a(z+aw)P_z-P_z\right\|_F\longrightarrow0.
\]
Consequently, there are $c_0,a_0>0$ such that
\begin{equation}
\operatorname{tr}\Sigma_a(y)\ge c_0a^2
\quad\text{whenever}\quad
\operatorname{dist}(y,\mathcal M)\le R_0a,\quad a\le a_0.
\label{eq:posterior-covariance-lower}
\end{equation}
For $Y_a=X_0+aZ$, the event in \eqref{eq:posterior-covariance-lower} has
probability at least $\mathbb P(\|Z\|\le R_0)>0$. Hence
\[
\operatorname{mmse}(a)
:=\ExpE\|X_0-D_a(Y_a)\|^2
=\ExpE\operatorname{tr}\Sigma_a(Y_a)
\ge c_1a^2.
\]
The conditional-variance identity also gives
\[
\ExpE\|X_0\|^2-\ExpE\|D_a(Y_a)\|^2=\operatorname{mmse}(a).
\]
Both $X_0$ and $D_a(Y_a)$ lie in the compact convex hull of $\mathcal M$.
Choose a globally Lipschitz function agreeing with $\|x\|^2$ on this convex
hull. Kantorovich--Rubinstein duality then yields
$\Wtwo\ge W_1\ge c_2\operatorname{mmse}(a)\ge c a^2$, completing the lower
bound.
\end{proof}

\begin{proof}[Proof of Theorem~\ref{thm:manifold-finite-floor}]
Let $Y=X_0+aZ$, set $r=\varepsilon/a$, and put
$\delta=(a^2-\varepsilon^2)^{1/2}$. The projection and posterior estimates in
the preceding proof give
\[
\operatorname{proj}_{\mathcal M}(Y)=X_0+aP_{X_0}Z+R_a,
\qquad
\|R_a\|_{L^2}\le Ca^2,
\]
and
$\|D_a(Y)-\operatorname{proj}_{\mathcal M}(Y)\|_{L^2}\le Ca^2$.
Since $T_{a\to\varepsilon}=r\,\mathrm{id}+(1-r)D_a$,
\begin{equation}
T_{a\to\varepsilon}(Y)
=X_0+aP_{X_0}Z+\varepsilon P_{X_0}^{\perp}Z
+\widetilde R_{a,\varepsilon},
\qquad
\|\widetilde R_{a,\varepsilon}\|_{L^2}
\le C(1-r)a^2\le C\delta^2.
\label{eq:finite-floor-expansion}
\end{equation}
Conditionally on $X_0$,
\[
aP_{X_0}Z+\varepsilon P_{X_0}^{\perp}Z
\stackrel{d}{=}
\delta P_{X_0}Z_1+\varepsilon Z_0,
\]
where $Z_0,Z_1$ are independent standard ambient Gaussians. Couple
$X_0+\delta P_{X_0}Z_1$ to intrinsic Brownian motion at time $\delta^2$
using the ambient It\^o construction above. Its $L^2$ remainder is
$O(\delta^2)$. Gaussian convolution is nonexpansive in $W_2$, and the
short-time heat estimate gives
\[
\Wtwo\big((H_{\delta^2}p_0)*\mathcal N(0,\varepsilon^2I),
p_0*\mathcal N(0,\varepsilon^2I)\big)
\le\Wtwo(H_{\delta^2}p_0,p_0)\le C\delta^2.
\]
Together with \eqref{eq:finite-floor-expansion}, this proves
\eqref{eq:manifold-finite-floor}. The estimate for an input law $q$ follows
by pushing an optimal coupling of $q$ and $p_a$ through the globally
$K_{a,\varepsilon}$-Lipschitz map and applying the triangle inequality.
\end{proof}

\begin{remark}[Boundary sharpness]
The boundaryless assumption is rate-sharp. Let $p_0$ be uniform on $[0,1]$
and let $\varphi,\Phi$ denote the standard normal density and distribution
function. Scaling the two boundary layers by $y=as$ and using one-dimensional
monotone transport gives
\[
\Wtwo((D_a)_\#p_a,p_0)
=c_{\partial}a^{3/2}+o(a^{3/2}),
\]
where
\[
c_{\partial}^2
=2\int_{-\infty}^{\infty}
\left[s+\frac{\varphi(s)}{\Phi(s)}-s\Phi(s)-\varphi(s)\right]^2
\Phi(s)\,\dd s,
\qquad
c_{\partial}=0.5494442783\ldots .
\]
Thus the quadratic conclusion does not extend unchanged to manifolds with
boundary.
\end{remark}

\section{Normal-mode clocks and the expanding-horizon obstruction}
\label{app:clock-classification}

On a normal coordinate, $D=0$ and
\[
\frac{\dd x}{\dd\sigma}=\frac{x}{\sigma},
\qquad x(\sigma)=\frac{\sigma}{s}x(s).
\]
For a $C^1$ monotone clock $\phi(\sigma)$ with $\phi'(\sigma)\ne0$, Euler from
$s$ to $t=se^{-\ell}$ has multiplier
\[
1-\psi_\phi(s,\ell),
\qquad
\psi_\phi(s,\ell)
=\frac{\phi(s)-\phi(se^{-\ell})}{s\phi'(s)}.
\]

\begin{proposition}[Normal-mode-exact Euler clocks]
\label{prop:layerexact}
Clocked Euler reproduces the normal-mode map on every step if and only if
$\phi$ is affine in $\sigma$. Hence the $\sigma$-clock is the unique
normal-mode-exact one-stage Euler clock up to affine reparameterization.
\end{proposition}

\begin{proof}
Exactness is
$\phi(s)-\phi(se^{-\ell})=s\phi'(s)(1-e^{-\ell})$ for every $s,\ell$.
Differentiation in $\ell$ gives
$\phi'(se^{-\ell})=\phi'(s)$, so $\phi'$ is constant. The converse follows
by substitution.
\end{proof}

For a uniform $\lambda=\log(\smax/\sigma)$ step bounded by $h$, telescoping
gives $N\ge\log(\smax/\varepsilon)/h$; this proves the first row of
Table~\ref{tab:classification}. If instead $N$ is fixed while
$\varepsilon\downarrow0$, a uniform $\lambda$-mesh has
$\ell=\log(\smax/\varepsilon)/N\to\infty$. For $\lambda$-Euler the normal output
multiplier is $(1-\ell)^N$, so its second moment diverges. This proves the
second row.

\begin{proposition}[Power-clock Euler]
\label{prop:power-clock}
Let $\tau=\sigma^\gamma$ with $\gamma\ge1$, set
$N=\lceil\smax^\gamma/h\rceil$, and use the uniform $\tau$-mesh between
$\smax^\gamma$ and $\varepsilon^\gamma$. On the normal-mode test equation, Euler is exact for
$\gamma=1$. For $\gamma>1$ it is AP and has sharp floor-uniform order
$1/\gamma$. On a fixed
Gaussian with at least one positive and one zero covariance eigenvalue, the
full endpoint law has the same sharp orders: one for $\gamma=1$ and
$1/\gamma$ for $\gamma>1$. If every covariance eigenvalue is zero, the
$\gamma=1$ update is exact.
\end{proposition}

\begin{proof}
Write $\alpha=1/\gamma$ and
\[
\Delta_\varepsilon=\frac{\smax^\gamma-\varepsilon^\gamma}{N}\le h,
\qquad \rho=\frac{\varepsilon^\gamma}{\Delta_\varepsilon},
\qquad \tau_j=\varepsilon^\gamma+j\Delta_\varepsilon.
\]
The normal equation becomes $\dd x/\dd\tau=\alpha x/\tau$, so every
decreasing Euler multiplier is $1-\alpha/(\rho+j)\in[0,1]$. The update
count and the second-moment bound are therefore independent of the floor.
The endpoint standard deviation is
\[
s_{h,\gamma}(\varepsilon)=\Delta_\varepsilon^\alpha(\rho+N)^\alpha
\frac{\Gamma(\rho+N+1-\alpha)\Gamma(\rho+1)}
{\Gamma(\rho+N+1)\Gamma(\rho+1-\alpha)}.
\]
For $0<\alpha<1$, Gamma log-convexity and concavity of $x^\alpha$ give
\[
0\le s_{h,\gamma}(\varepsilon)-\varepsilon
\le (1-\alpha)^\alpha h^\alpha.
\]
At zero floor,
\[
s_{h,\gamma}(0)=\smax
\frac{\Gamma(N+1-\alpha)}{\Gamma(1-\alpha)\Gamma(N+1)}
\sim\frac{\Delta_0^\alpha}{\Gamma(1-\alpha)},
\]
which proves sharpness. When $\gamma=1$, every multiplier equals its analytic
counterpart, including the last zero-floor multiplier.

For a positive Gaussian eigenvalue $c$, the power-clock coefficient is
\[
b_c(\tau)=\frac1\gamma
\frac{\tau^{2/\gamma-1}}{c+\tau^{2/\gamma}}.
\]
The right-sum error and the logarithmic Euler correction are both
$O(h^{\min\{1,2/\gamma\}})$ uniformly in the floor: bounded variation gives
$O(h)$ for $\gamma\le2$, while for $\gamma>2$ the omitted first-cell integral
is $O(h^{2/\gamma})$. This is always higher order than the normal error when
$\gamma>1$. For $\gamma=1$ and zero floor, write
$b_c(\tau)=\tau/(c+\tau^2)$. The difference between the numerical and exact
log multipliers is
\[
-\frac{\Delta_0}{2}\int_0^{\smax}\frac{c}{(c+\tau^2)^2}\,\dd\tau
+O(\Delta_0^2),
\]
because $b_c'+b_c^2=c/(c+\tau^2)^2$. This coefficient is nonzero for every
$c>0$; since $\Delta_0\sim h$, this proves first-order sharpness when a
positive mode is present.
Orthogonal Gaussian tensorization completes the proof.
\end{proof}

\section{Gaussian uniform-accuracy orders for representative samplers}
\label{app:gaussian-orders}

Let $p_0=\mathcal N(0,C)$ with positive eigenvalues
$c_1,\ldots,c_d$ and zero eigenvalues on the normal subspace. The PF-ODE
diagonalizes as
\[
\frac{\dd x_i}{\dd\sigma}
=\frac{\sigma}{c_i+\sigma^2}x_i,
\qquad
x_i(\sigma)=x_i(\smax)
\sqrt{\frac{c_i+\sigma^2}{c_i+\smax^2}}.
\]

\subsection{Fitted sigma-Euler: DDIM, DPM-Solver-1, and EDM Euler}

For one step $s\to t$ on a positive eigenvalue, the numerical-to-exact
standard-deviation ratio is
\[
\rho(s,t;c)=\frac{c+st}{\sqrt{(c+s^2)(c+t^2)}}\le1,
\]
because
$(c+s^2)(c+t^2)-(c+st)^2=c(s-t)^2$. On a decreasing mesh with
$\ell_n=\log(\sigma_n/\sigma_{n+1})\le h\le1$, set
$u_n=\sigma_n\sigma_{n+1}/c$. A Taylor expansion gives
\[
-\log\rho_n
=\frac{\ell_n^2}{2}\frac{u_n}{(1+u_n)^2}+O(\ell_n^3)
\frac{u_n}{(1+u_n)^2}.
\]
The local-error weight is integrable over log noise:
\[
\int_{-\infty}^{\infty}\frac{e^y}{(1+e^y)^2}\,\dd y=1.
\]
Therefore $-\sum_n\log\rho_n=O(h)$ uniformly in the floor and in every
positive eigenvalue. Normal modes are exact by
Proposition~\ref{prop:layerexact}. This proves first-order Gaussian UA for
the fitted $\sigma$-Euler row.

\subsection{Midpoint DPM-Solver-2}

For one step $s\to t=se^{-\ell}$, set $m=se^{-\ell/2}$. In VE
coordinates, Algorithm~1 of \citet{lu2022dpmsolver} reads
\[
U=x+(m-s)\eta(x,s),
\qquad
x^+=x+(t-s)\eta(U,m).
\]
For a positive Gaussian eigenvalue $c$, put $u=s^2/c$ and
$q=e^{-\ell/2}$. The numerical and exact scalar multipliers are
\[
\Psi_2
=1-\frac{uq(1-q^2)(1+uq)}{(1+uq^2)(1+u)},
\qquad
\Psi_\star=\sqrt{\frac{1+uq^4}{1+u}}.
\]
Both are positive: after clearing the positive denominator of $\Psi_2$, its
numerator is $1+u(1+q^2-q+q^3)+u^2q^4$. With $y=u/(1+u)$, the log ratio is analytic on the
compactified domain $0\le y\le1$, $0\le\ell\le\ell_0$, and direct
differentiation gives
\[
\log\frac{\Psi_2}{\Psi_\star}
=\frac18y(1-y)\ell^3+O\!\left(y(1-y)\ell^4\right).
\]
The remainder is uniform because the log ratio vanishes at $y=0$ and $y=1$
and to third order at $\ell=0$. Along a decreasing mesh,
\[
y_n-y_{n+1}
=\frac{u_n(1-e^{-2\ell_n})}
{(1+u_n)(1+u_ne^{-2\ell_n})},
\qquad
\sum_n\ell_n y_n(1-y_n)\le C.
\]
Therefore
$\sum_n|\log(\Psi_{2,n}/\Psi_{\star,n})|\le Ch^2$ whenever
$\max_n\ell_n\le h$. The base-integration $W_2$ error is $O(h^2)$ uniformly
in its lower endpoint. On a normal mode, $\eta(x,s)=x/s$, so
$U=(m/s)x$ and the final multiplier is exactly $t/s$. A fitted terminal map
at switching scale $a=O(h)$ adds $O(a^2)$ tangential error and preserves the second
order at every lower floor.

\subsection{Sigma-Heun: EDM}

For $t=se^{-\ell}$ and $u=s^2/c$, the ratio of the Heun multiplier to the
exact multiplier is
\begin{equation}
R_H(u,\ell)
=\frac{2u^2+ue^{3\ell}+3ue^\ell+2e^{3\ell}}
{2\sqrt{u+1}(u+e^{2\ell})^{3/2}}.
\label{eq:heun-ratio}
\end{equation}
Differentiating at $\ell=0$ gives
\[
\log R_H(u,\ell)
=\frac{u^2}{2(1+u)^3}\ell^3
+O\!\left(\frac{u^2}{(1+u)^3}\ell^4\right),
\]
uniformly for $0\le\ell\le\ell_0<1$. Along log-noise time,
\[
\int_0^\infty\frac{u(\lambda)^2}{(1+u(\lambda))^3}\,\dd\lambda
=\frac12\int_0^\infty\frac{u}{(1+u)^3}\,\dd u=\frac14.
\]
Consequently $|\sum_n\log R_H|\le Ch^2$. Normal modes are exact because
both Heun stage derivatives equal the conserved $x/\sigma$; EDM's skipped
terminal correction leaves the same exact Euler step at zero. A switching scale
$a=O(h)$ adds $O(a^2)$ by Proposition~\ref{prop:gaussian-closure}, proving
second-order Gaussian UA.

\subsection{Affine Rectified Flow}

For $X_t=(1-t)X_0+tZ$, set
$t_{\max}=\smax/(1+\smax)<1$, so the VE ratio at the initial endpoint is
$t_{\max}/(1-t_{\max})=\smax$. A Gaussian coordinate has variance
$V_c(t)=(1-t)^2c+t^2$ and exact velocity
\[
v(t,x)=b_c(t)x,
\qquad
b_c(t)=\frac{t-(1-t)c}{V_c(t)}.
\]
For $c>0$, $b_c$ is smooth on the bounded interval $[0,t_{\max}]$, so the standard
local-error plus discrete-Gronwall argument gives classical global order
$p$ for an order-$p$ Runge--Kutta method, uniformly in the terminal floor.
For $c=0$, $x(t)=tZ$ has constant velocity $Z$. Internal consistency of the
Runge--Kutta stages makes every positive-time stage exact. If a stage would evaluate
the singular formula at zero, stop at native time $a_t>0$ and append the
fitted terminal map. Its VE switching scale is
$a_\sigma=a_t/(1-a_t)$. The condition
$a_\sigma=O(h^{p/2})$, equivalently $a_t=O(h^{p/2})$ as $h\to0$, preserves
order $p$ by the Gaussian completion estimate \eqref{eq:gaussian-closure}; a switching scale
$a_t\asymp h$ caps $p>2$ at second order.

\subsection{Euler--Maruyama}

In log-noise time, one Gaussian coordinate obeys
\[
\dd X=-(1+\beta)\frac{\sigma^2}{c+\sigma^2}X\,\dd\lambda
+\sqrt{2\beta}\,\sigma\,\dd B_\lambda.
\]
An Euler--Maruyama step of length $\ell_n$ gives the variance recurrence
\[
V_{n+1}=A_n^2V_n+2\beta\ell_n\sigma_n^2,
\qquad
A_n=1-(1+\beta)\ell_n\frac{\sigma_n^2}{c+\sigma_n^2}.
\]
Substitution of the exact variance $c+\sigma_n^2$ shows that the one-step
variance defect is $O(\ell_n^2\sigma_n^2)$, uniformly in
$\sigma_n^2/c$. For sufficiently small $h=\max_n\ell_n$, $|A_n|\le1$,
and
\[
|V_N-(c+\sigma_N^2)|
\le C\sum_n\ell_n^2\sigma_n^2
\le Ch\sum_n\ell_n\sigma_n^2\le Ch.
\]
For $c>0$, this is first-order $W_2$ error between the scalar Gaussian laws.
For $c=0$ and fixed $\beta>0$, the relative variance
$R_n=V_n/\sigma_n^2$ satisfies
\[
|R_{n+1}-1|
\le(1-c_0\ell_n)|R_n-1|+C\ell_n^2,
\]
for $h\le h_0(\beta)$, and hence $\sup_n|R_n-1|\le Ch$. The fitted map contracts these errors and
completes the terminal interval. This proves first-order UA for the Gaussian
endpoint laws.

\begin{proof}[Proof of Theorem~\ref{thm:sampler-classification}]
The first two rows are the step-count and stability counterexamples in
Appendix~\ref{app:clock-classification}. The third row follows from the
power-clock calculation in that appendix. The fourth through sixth rows follow
from the sigma-Euler, midpoint DPM-Solver-2, and sigma-Heun calculations
above. The seventh row is the Euler--Maruyama covariance calculation, and the
eighth and ninth rows are
the affine Rectified Flow calculations. The first conditional row applies
Theorem~\ref{thm:fittedap} and \eqref{eq:gaussian-closure} to
a uniformly order-$p$ base solver on $[a,\smax]$. The final row combines a
uniformly third-order base solver with
Theorem~\ref{thm:fitted-order-recovery}. In every AP
row, the displayed uniform error bound supplies stability and consistency;
the explicit grid rule gives a floor-independent step count and a discrete
terminal limit.
\end{proof}

\section{Published protocols and analyzed sampler specifications}
\label{app:source-policies}

Table~\ref{tab:classification} separates update formulas from grids and
terminal rules because the source papers combine them in different ways.

\paragraph{DDIM and DPM-Solver.}
The deterministic DDIM process ends at the clean variable
\citep{song2021ddim}. DPM-Solver-1 is algebraically identical to DDIM
\citep[Section~4.1]{lu2022dpmsolver}. The theoretical DPM-Solver description
uses a zero endpoint, whereas its practical experiments stop at a positive
cutoff and explicitly omit final denoising. The table classifies the displayed
specification with the fitted terminal map; a positive-cutoff protocol is a
different specification.

\paragraph{EDM.}
EDM appends $\sigma_N=0$ after its positive schedule
\citep[Algorithm~1]{karras2022edm}. The Heun correction is skipped on that
last interval, so the terminal update is precisely fitted sigma-Euler. The
stochastic EDM algorithm uses the same terminal rule. To obtain an AP family,
the last positive level is interpreted as $a(h)\downarrow0$, and the same
formula is used for every requested $0\le\varepsilon\le a(h)$.

\paragraph{Higher-order diffusion solvers.}
Midpoint DPM-Solver-2 provides the two-stage update used in the proved
Gaussian row; its floor-uniform proof is in
Appendix~\ref{app:gaussian-orders}. DPM-Solver-3 has a third-order
fixed-interval result \citep{lu2022dpmsolver}. DEIS uses polynomial extrapolation or transformed
Runge--Kutta and Adams--Bashforth methods; its basic exponential integrator
specializes to DDIM for the VP parameterization \citep{zhang2023deis}. PNDM combines a DDIM transfer with
higher-order estimates of the noise field \citep{liu2022pndm}.
DPM-Solver++ uses data prediction, and UniPC supplies order-$p$ predictors
and higher-order correctors \citep{lu2022dpmsolverplusplus,zhao2023unipc}.
The remaining cited order results are fixed-interval statements, with $p$ restricted
to the order established for the chosen member. The AP and UA rows
in Table~\ref{tab:classification} apply after the explicit fitted terminal
completion and under the uniform base-solver hypotheses stated there.

\paragraph{Euler--Maruyama.}
Euler--Maruyama is the standard explicit baseline for the reverse-time SDE
\citep{song2021sde}. The table applies it only on $[a,\smax]$ and uses
the fitted deterministic map for the remaining terminal interval. Its stated
order is the endpoint-law covariance order proved in
Appendix~\ref{app:gaussian-orders}.

\paragraph{Rectified Flow.}
Rectified Flow uses the affine interpolation
$X_t=(1-t)X_0+tX_1$ and standard ODE integration
\citep{liu2023rectifiedflow}. For Gaussian noise $X_1=Z$, the noise amplitude
is affine in the bounded native time. The table starts at
$t_{\max}=\smax/(1+\smax)<1$, exactly matching the finite VE high-noise level
used throughout the analysis. This is why its Euler normal update is
normal-mode exact and why standard Runge--Kutta order applies on positive Gaussian
modes.

\section{Learned-checkpoint switching-scale protocol}
\label{app:checkpoint-study}

\paragraph{Terminal-limit checks.}
A separate 64-seed sweep (seeds 3000--3063) varied the requested terminal
floor at fixed resolution. The fitted diffusion and bounded-native-time
Rectified Flow configurations kept floor-independent update counts; their
largest endpoint second-moment ratios relative to the dense references were
$1.036$ and $0.973$,
and the log--log slopes of same-seed RMS distance to the zero-floor output
against $\varepsilon$ were $1.000$ and $0.999$ at fixed resolution.
For bounded log-noise stepping without a terminal map, the regression-estimated
count slope against
$\log(\smax/\varepsilon)$ was $5.0099$ at $h=0.2$, compared with $1/h=5$.

The rate-transfer study uses 64 paired seeds (4000--4063) for each checkpoint.
For EDM, the step targets are
$\{0.28,0.20,0.14,0.10\}$ and the switching scales are
$\{0.40,0.20,0.10,0.05\}$. Rectified Flow uses step targets
$\{1/12,1/18,1/27,1/40\}$ and switching scales
$\{0.10,0.05,0.025,0.0125\}$ in its native coordinate. Dense integrations
with 1024 and 2048 steps provide reference states from the same checkpoint
and initial noise.

At each grid point we record the two brackets in
\eqref{eq:checkpoint-decomposition}, their full sum, and their cross term.
Writing these brackets as $U_{\rm base}$ and $U_{\rm term}$, their paired
root-mean-square errors satisfy
\[
E_{\rm end}^2=E_{\rm base}^2+E_{\rm term}^2
+2\rho E_{\rm base}E_{\rm term},
\qquad
\rho=\frac{\ExpE\langle U_{\rm base},U_{\rm term}\rangle}
{E_{\rm base}E_{\rm term}}.
\]
We fit $E_{\rm base}(h,a)=K_b h^p a^{-r}$ and
$E_{\rm term}(a)=K_m a^\nu$, together with the correlation $\rho$, on the coarsest
$3\times3$ subgrid. This gives nine fitting configurations; the union of the
finest-step row and smallest-switching-scale column gives seven held-out
configurations. The estimated parameters, with 5--95\% paired-bootstrap
intervals, are
\begin{quote}
\small\raggedright
EDM: $(p,r,\nu,\rho)=(1.967,0.008,0.757,-0.182)$, with intervals
$[1.929,2.011]$, $[-0.009,0.031]$, $[0.740,0.773]$, and
$[-0.207,-0.161]$, respectively.\par
Rectified Flow: $(p,r,\nu,\rho)=(1.926,-0.008,0.898,-0.052)$, with intervals
$[1.576,2.365]$, $[-0.013,-0.001]$, $[0.881,0.914]$, and
$[-0.067,-0.036]$, respectively.
\end{quote}
The empirical regression permits a signed effective exponent. The Rectified
Flow estimate $r=-0.008$ indicates an essentially switching-scale-independent
base-error amplitude over the sampled range and corresponds to the $r=0$
branch of Theorem~\ref{thm:nfe-switching}.
The finest-step row and smallest-switching-scale column are predicted without
refitting. Intervals
come from 2000 paired-seed bootstrap replicates. The algebraic decomposition
holds to relative error below $5\times10^{-16}$, and the 1024--2048 reference
differences are below $8.0\times10^{-6}$ for EDM and
$2.5\times10^{-5}$ for Rectified Flow.
Across all 16 leave-one-step/leave-one-switching-scale sensitivity splits, the
estimated base-integration order remains above $1.92$ for both checkpoints and
the terminal-map exponent
remains in $[0.757,0.823]$ for EDM and $[0.898,1.006]$ for Rectified
Flow. Prediction errors vary more across these sensitivity splits for Rectified
Flow; the reported held-out split was fixed before the final runs.

\begin{quote}
\small\raggedright
\textbf{EDM.} Public CIFAR-10 unconditional VP checkpoint
\nolinkurl{edm-cifar10-32x32-uncond-vp.pkl}\footnote{\href{https://nvlabs-fi-cdn.nvidia.com/edm/pretrained/edm-cifar10-32x32-uncond-vp.pkl}{Official checkpoint download.}},
released with EDM \citep{karras2022edm}.\par
\textbf{Rectified Flow.} Configuration
\nolinkurl{configs/rectified_flow/cifar10_rf_gaussian_ddpmpp.py}, training step
800001, commit \nolinkurl{5a1fd4dd3ea7db764ce370a84ce35f9c8b15fde6}, and
checkpoint SHA-256
\texttt{27d1463f573556765d380b1983664d24e}\allowbreak
\texttt{00a194853e0778f5af18fcdf34500de}.
\end{quote}
The implementation is validated against a closed-form synthetic decomposition.

\end{document}